%% file: uniform_icml.tex
\documentclass{article}
\usepackage{utlis}

\usepackage[accepted]{icml2021}

\icmltitlerunning{Generalization Error of Random Features Models}

\begin{document}

\twocolumn[
\icmltitle{Exact Gap between Generalization Error and Uniform Convergence in Random Feature Models}

\begin{icmlauthorlist}
\icmlauthor{Zitong Yang}{eecs}
\icmlauthor{Yu Bai}{sf}
\icmlauthor{Song Mei}{stats}
\end{icmlauthorlist}

\icmlaffiliation{eecs}{Department of Electrical Engineering and Computer Sciences, University of California, Berkeley.}
\icmlaffiliation{sf}{Salesforce Research.}
\icmlaffiliation{stats}{Department of Statistics, University of California, Berkeley}
\icmlcorrespondingauthor{Zitong Yang}{zitong@berkeley.edu}
\icmlcorrespondingauthor{Yu Bai}{yu.bai@salesforce.com}
\icmlcorrespondingauthor{Song Mei}{songmei@berkeley.edu}
\vskip 0.3in]
\printAffiliationsAndNotice{}

\input{sections/abstract.tex}

\input{sections/introduction.tex}

\input{sections/insights.tex}

\input{sections/mainthm.tex}

\input{sections/discussions.tex}

\newpage

\bibliography{reference}
\bibliographystyle{icml2021}

\clearpage
\appendix
\onecolumn
\input{sections/additional_material.tex}
\input{sections/proofs.tex}

\end{document}

%% file: sections/abstract.tex
\begin{abstract}
Recent work showed that there could be a large gap between the classical uniform convergence bound and the actual test error of zero-training-error predictors (interpolators) such as deep neural networks. 
To better understand this gap, we study the uniform convergence in the nonlinear random feature model and perform a precise theoretical analysis on how uniform convergence depends on the sample size and the number of parameters.
We derive and prove analytical expressions for three quantities in this model: 1) classical uniform convergence over norm balls, 2) uniform convergence over interpolators in the norm ball (recently proposed by~\citet{zhou2021uniform}), and 3) the risk of minimum norm interpolator. 
We show that, in the setting where the classical uniform convergence bound is vacuous (diverges to $\infty$), uniform convergence over the interpolators still gives a non-trivial bound of the test error of interpolating solutions. 
We also showcase a different setting where classical uniform convergence bound is non-vacuous, but uniform convergence over interpolators can give an improved sample complexity guarantee.
Our result provides a first exact comparison between the test errors and uniform convergence bounds for interpolators beyond simple linear models.
\end{abstract}

%% file: sections/introduction.tex
\section{Introduction}\label{sec:introduction}

\begin{figure*}[ht]
  \begin{center}
    \subfigure[Noiseless response ($\tau^2=0$)]{
    \includegraphics[width=.31\textwidth]{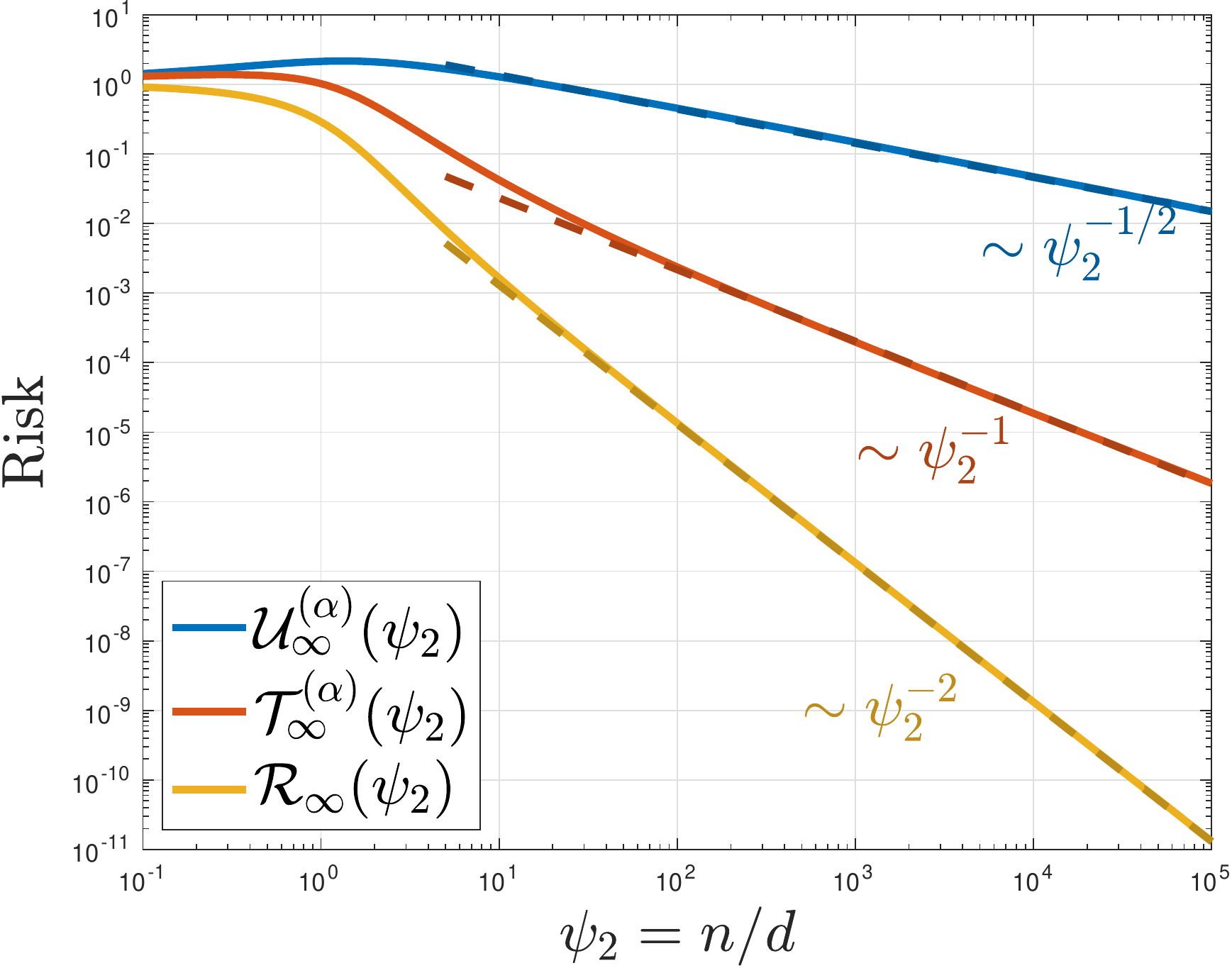}
    \label{fig:kernel_noiseless_utr}
    }
    \subfigure[Noisy response ($\tau^2=0.1$)]{
    \includegraphics[width=.31\textwidth]{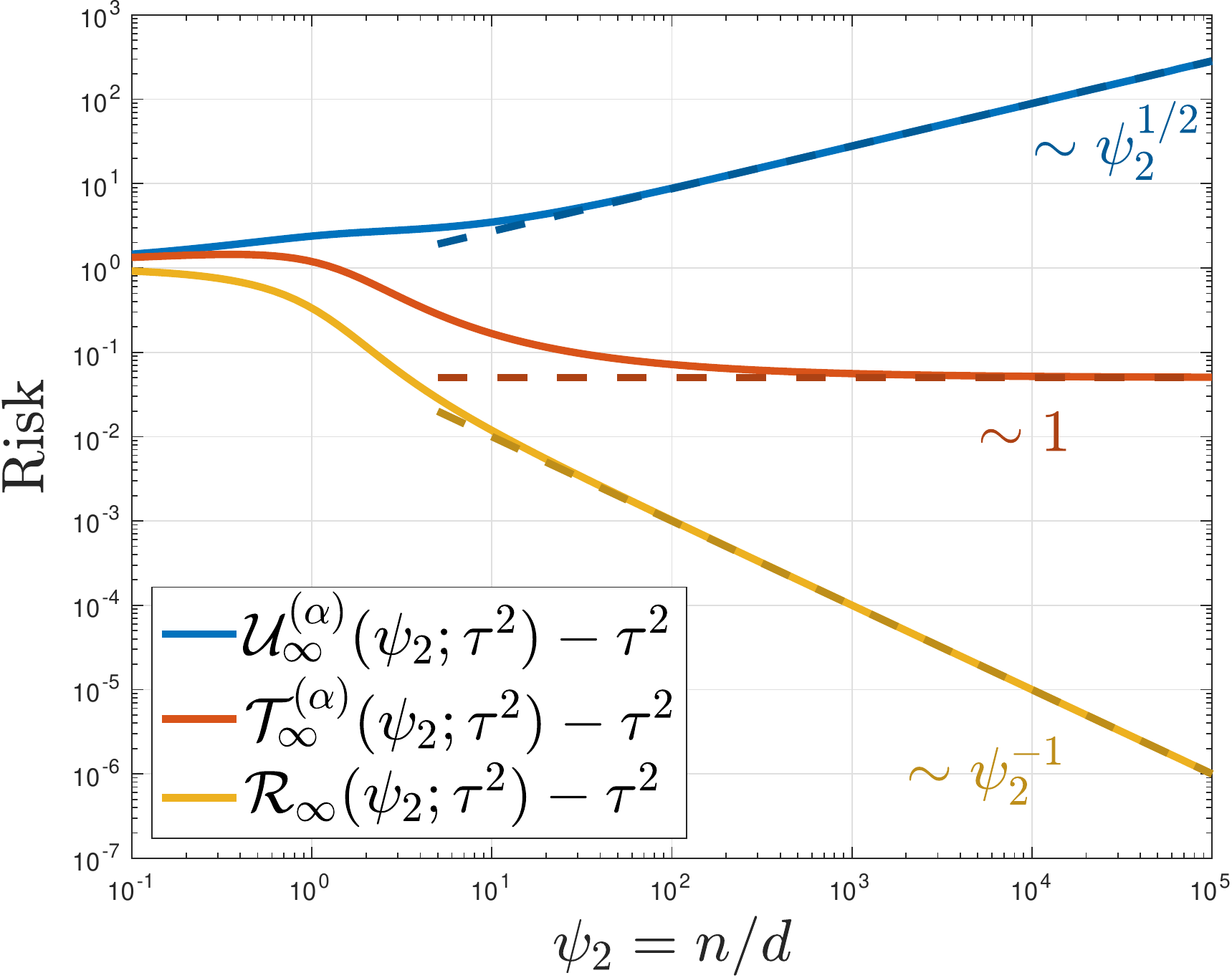}
    \label{fig:kernel_noisy_utr}
    }
    \subfigure[Minimum norm $\cA_\infty(\psi_2)$]{
    \includegraphics[width=.31\textwidth]{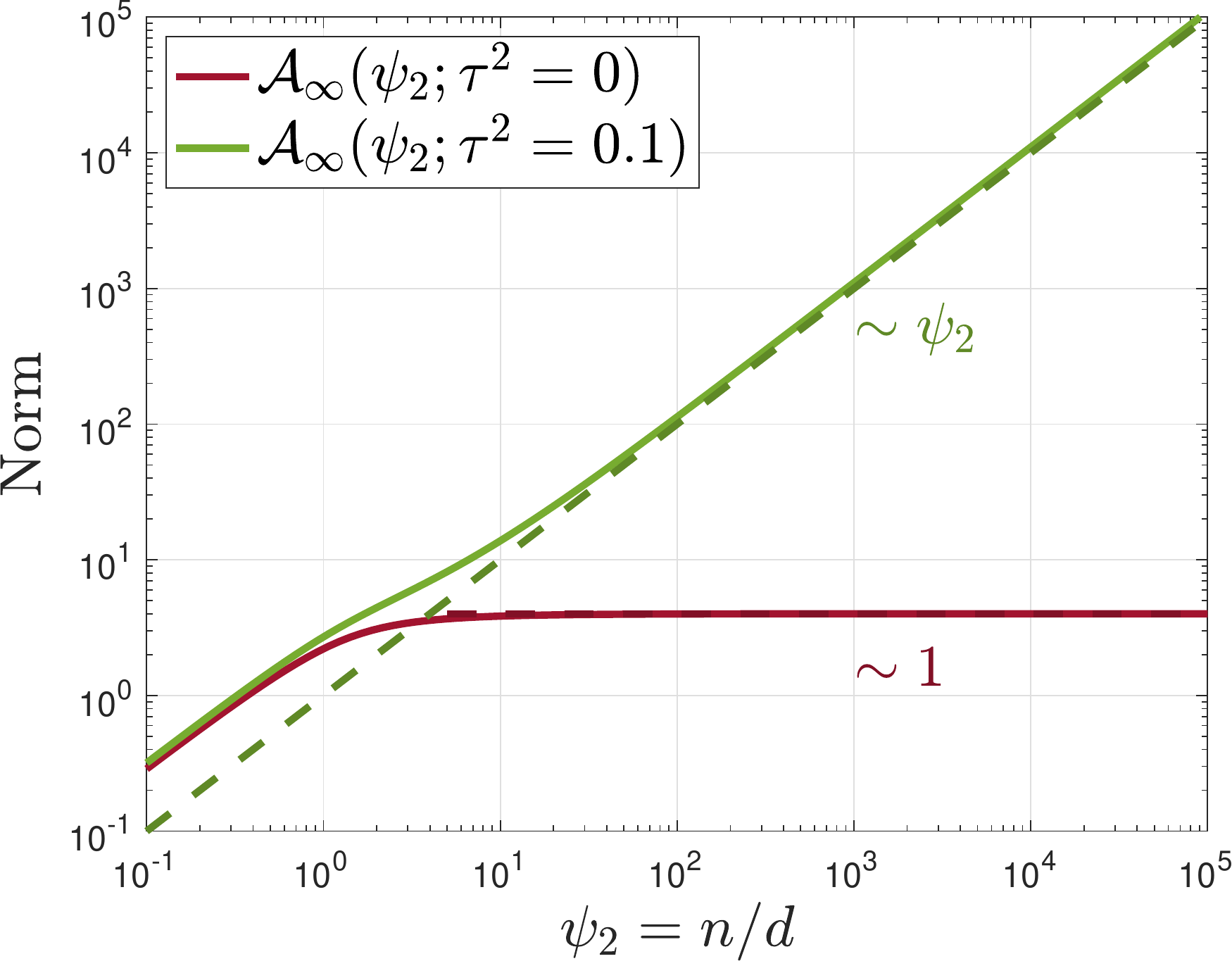}
    \label{fig:kernel_norm}
    }
    \vskip -0.1in
    \caption{Random feature regression with activation function $\sigma(x) = \max(0, x)- 1/\sqrt{2\pi}$, target function $f_d(\bx)=\<\bbeta, \bx\>$ with $\|\bbeta\|_2^2=1$, and $\psi_1=\infty$. The horizontal axes are the number of samples $\psi_2 = \lim_{d \to \infty} n/d$. The solid lines are the the algebraic expressions derived in the main theorem (Theorem \ref{thm:main_theorem}). The dashed lines are the function $\psi_2^p$ in the log scale. 
    Figure \ref{fig:kernel_noiseless_utr} and \ref{fig:kernel_noisy_utr}: Comparison of the classical uniform convergence in the norm ball of size level $\alpha = 1.5$ (Eq. \eqref{eqn:cU_inf_alpha}, blue curve), the uniform convergence over interpolators in the same norm ball (Eq. \eqref{eqn:cT_inf_alpha}, red curve), the risk of minimum norm interpolator (Eq. \eqref{eqn:cR_inf}, yellow curve).
    Figure \ref{fig:kernel_norm}: Minimum norm required to interpolate the training data (Eq. \eqref{eqn:cA_inf}).}
  \label{fig:kernel_noiseless}
  \end{center}
  \vskip -0.2in
\end{figure*}

Uniform convergence—the supremum difference between the training and test errors over a certain function class—is a powerful tool in statistical learning theory for understanding the generalization performance of predictors. 
Bounds on uniform convergence usually take the form of $\sqrt{\text{complexity}/n}$ \cite{vapnik95}, where the numerator represents the complexity of the function class, and $n$ is the sample size. If such a bound is tight, then the predictor is not going to generalize well whenever the function class complexity is too large.


However, it is shown in recent theoretical and empirical work that overparametized models such as deep neural networks could generalize well, even in the interpolating regime in which the model exactly memorizes the data~\citep{zhang2016understanding,belkin2019reconciling}. 
As interpolation (especially for noisy training data) usually requires the predictor to be within a function class with high complexity, this challenges the classical methodology of using uniform convergence to bound generalization.
For example, \citet{belkin2018understand} showed that interpolating noisy data with kernel machines requires exponentially large norm in fixed dimensions.
The large norm would effectively make the uniform convergence bound $\sqrt{\text{complexity}/n}$ vacuous.
\citet{zico2019unable} empirically measured the spectral-norm bound in~\citet{NIPS2017_b22b257a} and find that for interpolators, the bound increases with $n$, and is thus vacuous at large sample size. Towards a more fine-grained understanding, we ask the following
\begin{quote}
  {\bf Question}: How large is the gap between uniform convergence and the actual generalization errors for interpolators?
\end{quote}

In this paper, we study this gap in the random features model from \citet{aliben}. This model can be interpreted as a linearized version of two-layer neural networks \citep{jacot2018neural} and exhibit some similar properties to deep neural networks such as double descent \citep{belkin2019reconciling}. We consider two types of uniform convergence in this model:
\begin{itemize} 
  \item $\cU:$ The classical uniform convergence over a norm ball of radius $\sqrt{A}$.
  \item $\cT:$ The modified uniform convergence over the same norm ball of size $\sqrt{A}$ but only include the interpolators, proposed in \citet{zhou2021uniform}.
\end{itemize}
Our main theoretical result is the exact asymptotic expressions of two versions of uniform convergence $\cU$ and $\cT$ in terms of the number of features, sample size, as well as other relevant parameters in the random feature model.
Under some assumptions, we prove that the actual uniform convergence concentrates to these asymptotic counterparts. To further compare these uniform convergence bounds with the actual generalization error of interpolators, we adopt
\begin{itemize}
	  \item $\cR:$ the generalization error (test error) of the minimum norm interpolator.
\end{itemize}
from \citet{mm19}.
To make $\cU$, $\cT$, $\cR$ comparable with each other, we choose the radius of the norm ball $\sqrt{A}$ to be slightly larger than the norm of the minimum norm interpolator.
Our limiting $\cU$, $\cT$ (with norm ball of size $\sqrt{A}$ as chosen above), and $\cR$ depend on two main variables: $\psi_1=\lim_{d\to\infty} N/d$ representing the number of parameters, and $\psi_2=\lim_{d\to\infty}  n/d$ representing the sample size.
Our formulae for $\cU, \cT$ and $\cR$ yield three major observations.

\begin{enumerate}
\item \textbf{Sample Complexity in the Noisy Regime:}
  When the training data contains label noise (with variance $\tau^2$), we find that the norm required to interpolate the noisy training set grows linearly with the number of samples $\psi_2$ (green curve in Figure \ref{fig:kernel_norm}). As a result, the standard uniform convergence bound $\cU$ grows with $\psi_2$ at the rate $\cU\sim\psi_2^{1/2}$, leading to a vacuous bound on the generalization error (Figure \ref{fig:kernel_noisy_utr}). 
  
  In contrast, in the same setting, we show the uniform convergence over interpolators $\cT\sim 1$ is a constant for large $\psi_2$, and is only order one larger than the actual generalization error $\cR\sim 1$. Further, the excess versions scale as $\cT-\tau^2\sim 1$ and $\cR-\tau^2\sim \psi_2^{-1}$.

\item \textbf{Sample Complexity in the Noiseless Regime:} When the training set does not contain label noise, the generalization error $\cR$ decays faster: $\cR\sim\psi_2^{-2}$. In this setting, we find that the classical uniform convergence $\cU\sim\psi_2^{-1/2}$ and the uniform convergence over interpolators $\cT\sim\psi_2^{-1}$. This shows that, even when the classical uniform convergence already gives a non-vacuous bound, there still exists a sample complexity separation among the classical uniform convergence $\cU$, the uniform convergence over interpolators $\cT$, and the actual generalization error $\cR$.

\item \textbf{Dependence on Number of Parameters:}
  In addition to the results on $\psi_2$, we find that $\cU, \cT$ and $\cR$ decay to its limiting value at the same rate $1/\psi_1$.
  This shows that both $\cU$ and $\cT$ correctly predict that as the number of features $\psi_1$ grows, the risk $\cR$ would decrease.  
\end{enumerate}

These results provide a more precise understanding of uniform convergence versus the actual generalization errors, under a natural model that captures a lot of essences of nonlinear overparametrized learning.

\subsection{Related work}
\paragraph{Classical theory of uniform convergence.}
Uniform convergence dates back to the empirical process theory of \citet{givenko1933} and \citet{cantelli1933}. Application of uniform convergence to the framework of empirical risk minimization usually proceeds through Gaussian and Rademacher complexities \cite{bartlettAndMendelson, bartlett2005} or VC and fat shattering dimensions \cite{vapnik95, bartlett1998the}. 

\paragraph{Modern take on uniform convergence.}
A large volume of recent works showed that overparametrized interpolators  could generalize well \cite{zhang2016understanding, belkin18a, behnm2015search, advani2020high, bartlett2020benign,belkin2018overfit, pmlr-v89-belkin19a, Nakkiran2020Deep, yang2020rethinking, belkin2019reconciling, mm19,spigler2019jamming}, suggesting that the classical uniform convergence theory may not be able to explain generalization in these settings \cite{zhang2016understanding}. Numerous efforts have been made to remedy the original uniform convergence theory using the Rademacher complexity \cite{pmlr-v40-Neyshabur15, pmlr-v75-golowich18a, neyshabur2018the,NIPS2009_f7664060,NEURIPS2019_cf9dc5e4}, the compression approach \cite{pmlr-v80-arora18b}, covering numbers \cite{NIPS2017_b22b257a}, derandomization \cite{negrea2020defense} and PAC-Bayes methods \cite{dziugaite2017computing,neyshabur2018a,nagarajan2018deterministic}. Despite the progress along this line, \citet{zico2019unable,bartlett2021failures} showed that in certain settings ``any uniform convergence'' bounds cannot explain generalization. Among the pessimistic results, \citet{zhou2021uniform} proposes that uniform convergence over interpolating norm ball could explain generalization in an overparametrized linear setting. 
Our results show that in the nonlinear random feature model, there is a sample complexity gap between the excess risk and uniform convergence over interpolators proposed in \citet{zhou2021uniform}.

\paragraph{Random features model and kernel machines.}
A number of papers studied the generalization error of kernel machines
\cite{caponnetto2007optimal, jacot2020kernel, wainwright2019high} and random features models
\cite{rahimi2009weighted, rudi2017generalization, bach2015equivalence, ma2020towards} in the non-asymptotic settings, in which the generalization error bound depends on the RKHS norm. However, these bounds cannot characterize the generalization error for interpolating solutions. In the last three years, a few papers \cite{belkin2018understand, liang2020just, liang2019risk} showed that interpolating solutions of kernel ridge regression can also generalize well in high dimensions. 
Recently, a few papers studied the generalization error of random features model in the proportional asymptotic limit in various settings \cite{hastie2019surprises, louart2018random, mm19, montanari2019generalization, gerace2020generalisation, d2020double, yang2020rethinking, adlam2020understanding,dhifallah2020precise,hu2020universality}, where they precisely characterized the asymptotic generalization error of interpolating solutions, and showed that double-descent phenomenon \cite{belkin2019reconciling,advani2020high} exists in these models. A few other papers studied the generalization error of random features models in the polynomial scaling limits \cite{ghorbani2019linearized, ghorbani2020neural, mei2021generalization}, where other interesting behaviors were shown. 

Precise asymptotics for the Rademacher complexity of some \emph{underparameterized} learning models was calculated using statistical physics heuristics in \citet{abbaras2020rademacher}. 
In our work, we instead focus on the uniform convergence of \emph{overparameterized} random features model.  


%% file: sections/insights.tex
\section{Problem formulation}\label{sec:model_setup_def_limit}
In this section, we present the background needed to understand the insights from our main result.
In Section \ref{sec:model_setup} we define the random feature regression task that this paper focuses on.
In Section \ref{sec:def_limit}, we informally present the limiting regime our theory covers.

\subsection{Model setup}\label{sec:model_setup}

Consider a dataset $(\bx_i, y_i)_{i \in [n]}$ with $n$ samples. Assume that the covariates follow $\bx_i \sim_{iid} \Unif(\S^{d-1}(\sqrt d))$, and responses satisfy $y_i = f_d(\bx_i) + \eps_i$, with the noises satisfying $\eps_i \sim_{iid} \cN(0, \tau^2)$ which are independent of $(\bx_i)_{i \in [n]}$. We will consider both the noisy ($\tau^2 > 0$) and noiseless ($\tau^2 = 0$) settings. 

We fit the dataset using the random features model. Let $(\btheta_j)_{j \in [N]} \sim_{iid} \Unif(\S^{d-1}(\sqrt d))$ be the random feature vectors. Given an activation function $\sigma: \R \to \R$, we define the random features function class $\cF_{\RF}(\bTheta)$ by
\[
\cF_{\RF}(\bTheta) \equiv \Big\{ f(\bx) = \sum_{j = 1}^N a_j \sigma\big(\< \bx, \btheta_j\>/\sqrt d \big): \ba \in \R^N \Big\}.
\]

\paragraph{Generalization error of the minimum norm interpolator.} Denote the population risk and the empirical risk of a predictor $\ba \in \R^N$ by
\begin{align}
R(\ba) =&~ \E_{\bx, y}~ \Big( y - \sum_{j = 1}^N a_j \sigma(\< \bx, \btheta_j\>/\sqrt d) \Big)^2,\label{eqn:pop_risk}\\
\what R_n(\ba) =&~ \frac{1}{n}\sum_{i = 1}^n \Big( y_i - \sum_{j = 1}^N a_j \sigma(\< \bx_i, \btheta_j\>/\sqrt d) \Big)^2, \label{eqn:emp_risk}
\end{align}
and the regularized empirical risk minimizer with vanishing regularization by
\[
\ba_{\min} = \lim_{\lambda \to 0+} \arg\min_{\ba} \Big[ \what R_n(\ba) + \lambda \|\ba\|_2^2 \Big].
\]
In the overparameterized regime ($N>n$), under mild conditions, we have $\min_{\ba} \what R_n(\ba) = \what R_n(\ba_{\min})=0$. In this regime, $\ba_{\min}$ can be interpreted as the minimum $\ell_2$ norm interpolator. 

A quantity of interest is the generalization error of this predictor, which gives (with a slight abuse of notation)
\begin{align}
R(N, n, d) \equiv  R(\ba_{\min}). \label{eqn:risk_min_l2}
\end{align}

\paragraph{Uniform convergence bounds.} We denote the uniform convergence bound over a norm ball and the uniform convergence over interpolators in the norm ball by
\begin{align}
&U(A, N, n, d) \equiv  \sup_{(N/d) \| \ba \|_2^2 \le A} \Big( R(\ba) - \what R_n(\ba) \Big),
\label{eqn:uniform}\\
&T(A, N, n, d) \equiv \sup_{(N/d) \| \ba \|_2^2 \le A, \what R_n(\ba) = 0} R(\ba). \label{eqn:uniform_zeroloss}
\end{align}
Here the scaling factor $N/d$ of the norm ball is such that the norm ball converges to a non-trivial RKHS norm ball with size $\sqrt{A}$ as $\psi_1 \to \infty$ (limit taken after $N/d \to \psi_1$). Note that in order for the maximization problem in \eqref{eqn:uniform_zeroloss} to have a non-empty feasible region, we need $\what R_n(\ba_{\min}) = 0$ and need to take $A \geq (N/d)\|\ba_{\min}\|_2^2$: we will show that in the region $N > n$ with sufficiently large $A$, this happens with high probability. 

By construction, for any $A\geq (N/d)\|\ba_{\min}\|_2^2$, we have $U(A) \ge T(A) \ge R(\ba_{\min})$ (see Figuire \ref{fig:simulation}). So a natural problem is to quantify the gap among $U(A)$, $T(A)$, and $R(\ba_{\min})$, which is our goal in this paper.  

\begin{figure}[ht]
  \begin{center}
    \includegraphics[width=0.35\textwidth]{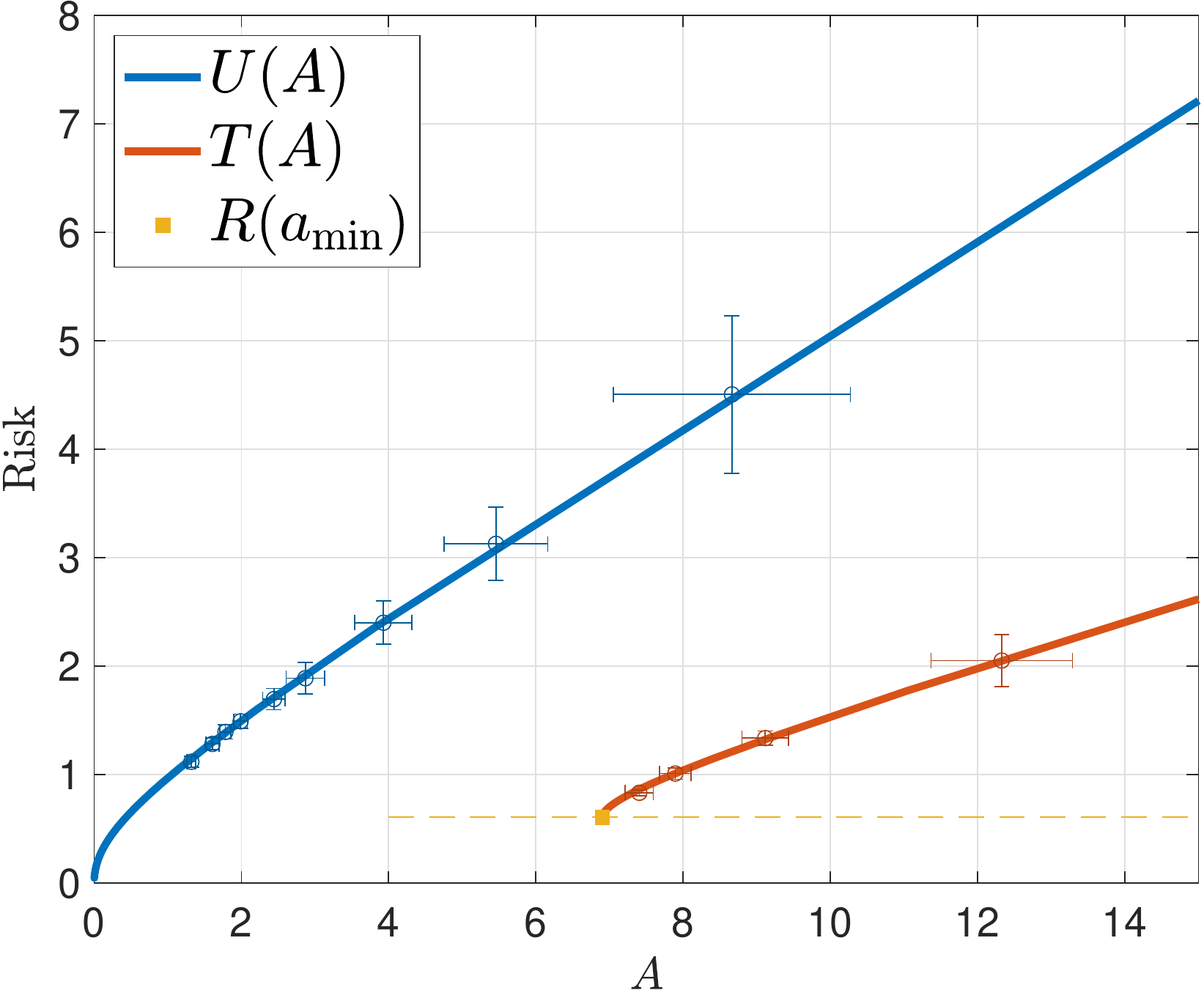}
    \vskip -0.1in
    \caption{
    Illustration of uniform convergence $U$ (c.f. eq. \eqref{eqn:uniform}), uniform convergence over interpolators $T$ (c.f. eq. \eqref{eqn:uniform_zeroloss}), and minimum norm interpolator $R(\ba_{\min})$. We take $y_i =\<\bx_i, \bbeta\>$ for some $\|\bbeta\|_2^2=1$, and take the ReLU activation function $\sigma(x) = \max\{ x, 0\}$. Solid lines are our theoretical predictions $\cU$ and $\cT$ (cf.~\eqref{eqn:u_limit} \&~\eqref{eqn:t_limit}). Points with error bars are obtained from simulations with the number of features $N=500$, number of samples $n=300$, and covariate dimension $d=200$. The error bar reports $1/\sqrt{20}\times$standard deviation over $20$ instances. See Appendix \ref{sec:simulation_detail} for details.
    }
  \label{fig:simulation}
  \end{center}
  \vskip -0.2in
\end{figure}

\subsection{High dimensional regime}\label{sec:def_limit}

We approach this problem in the limit $d \to \infty$ with $N/d \to \psi_1$ and $n/d \to \psi_2$ (c.f. Assumption \ref{ass:linear}). We further assume the setting of a linear target function $f_d$ and a nonlinear activation function $\sigma$ (c.f. Assumptions \ref{ass:linear_target} and \ref{ass:activation}). 
In this regime, our main result Theorem \ref{thm:main_theorem} will show that, the uniform convergence $U$ and the uniform convergence over interpolators $T$ will converge to deterministic functions, i.e., writing here informally, 
\begin{align}
U(A, N, n, d) \overset{d \to \infty}{\rightarrow} \cU(A, \psi_1, \psi_2), \label{eqn:u_limit}\\
T(A, N, n, d) \overset{d \to \infty}{\rightarrow}\cT(A, \psi_1, \psi_2), \label{eqn:t_limit}
\end{align}
where $\cU$ and $\cT$ will be defined in Definition~\ref{def:formula_U_T} (which depends on the definition of some other quantities that are defined in Appendix \ref{sec:analytic_expression_overline} and heuristically presented in Remark \ref{rmk:heuristic_def}). In addition to $\cU$ and $\cT$, Theorem 1 of \citet{mm19} implies the following convergence
\begin{align}
 (N/d) \|\ba_{\min}\|_2^2 \overset{d \to \infty}{\rightarrow} &~ \cA(\psi_1, \psi_2), \label{eqn:cA}\\
R(\ba_{\min}) \overset{d \to \infty}{\rightarrow} &~ \cR(\psi_1, \psi_2). \label{eqn:cR}
\end{align}
The precise algebraic expression of equation \eqref{eqn:cA} and \eqref{eqn:cR} was given in Definition 1 of \citet{mm19}, and we include in Appendix \ref{sec:analytic_expression_overline} for completeness. We will sometimes refer to $\cU, \cT, \cA, \cR$ without explicitly mark their dependence on $A, \psi_1, \psi_2$ for notational simplicity.

\paragraph{Kernel regime. }
\citet{aliben} have shown that, as $N \to \infty$, the random feature space $\cF_{\RF}(\bTheta)$ (equipped with proper inner product) converges to the RKHS (Reproducing Kernel Hilbert Space) induced by the kernel
\[
H(\bx, \bx') = \E_{\bw \sim \Unif(\S^{d-1})}[\sigma(\< \bx, \bw\>) \sigma(\< \bx', \bw)\>].
\]
We expect that, if we take limit $\psi_1\rightarrow\infty$ after $N, d, n\rightarrow\infty$, the formula of $\cU$ and $\cT$ will coincide with the corresponding asymptotic limit of $U$ and $T$ for kernel ridge regression with the kernel $H$. This intuition has been mentioned in a few papers \cite{mm19, d2020double, jacot2020implicit}. In this spirit, we denote

\begin{align}
\cU_\infty(A, \psi_2) \equiv&~ \lim_{\psi_1 \to \infty} \cU(A, \psi_1, \psi_2), 
\label{eqn:cU_inf}\\
\cT_\infty(A, \psi_2) \equiv&~ \lim_{\psi_1 \to \infty} \cT(A, \psi_1, \psi_2),
\label{eqn:cT_inf}\\
\cA_\infty(\psi_2) \equiv&~ \lim_{\psi_1 \to \infty} \cA(\psi_1, \psi_2),
\label{eqn:cA_inf}\\
\cR_\infty(\psi_2) \equiv&~ \lim_{\psi_1 \to \infty} \cR(\psi_1, \psi_2).
\label{eqn:cR_inf}
\end{align}

We will refer to the quantities $\{\cU_\infty, \cT_\infty, \cA_\infty, \cR_\infty\}$ as the $\{$uniform convergence in norm ball, uniform convergence over interpolators in norm ball, minimum $\ell_2$ norm of interpolators, and generalization error of interpolators$\}$ of kernel ridge regression. 

\paragraph{Low norm uniform convergence bounds.} There is a question of which norm $A$ to choose in $\cU$ and $\cT$ to compare with $\cR$. 
In order for $U$ and $T$ to serve as proper bounds for $R(\ba_{\min})$, we need to take at least $A \geq \psi_1\|\ba_{\min}\|_2^2$. Therefore, we will choose 
\begin{equation}\label{eqn:low_norm}
A = \alpha \psi_1 \|\ba_{\min}\|_2^2,
\end{equation}
for some $\alpha>1$ (e.g., $\alpha = 1.1$). 
Note $\psi_1 \|\ba_{\min}\|_2^2 \to \cA(\psi_1, \psi_2)$ as $d \to \infty$. So for a fixed $\alpha > 1$, we further define 
\begin{align}
\cU^{(\alpha)}(\psi_1, \psi_2) \equiv&~ \cU(\alpha \cA(\psi_1, \psi_2), \psi_1, \psi_2), 
\label{eqn:cU_alpha}\\
\cT^{(\alpha)}(\psi_1, \psi_2) \equiv&~ \cT(\alpha \cA(\psi_1, \psi_2), \psi_1, \psi_2), 
\label{eqn:cT_alpha}
\end{align}
and their kernel version,
\begin{align}
\cU^{(\alpha)}_\infty(\psi_2) \equiv&~ \lim_{\psi_1 \to \infty} \cU^{(\alpha)}(\psi_1, \psi_2),
\label{eqn:cU_inf_alpha}\\
\cT^{(\alpha)}_\infty(\psi_2) \equiv&~ \lim_{\psi_1 \to \infty} \cT^{(\alpha)}(\psi_1, \psi_2). 
\label{eqn:cT_inf_alpha}
\end{align}
This definition ensures that $\cR(\psi_1, \psi_2) \le \cT^{(\alpha)}(\psi_1, \psi_2) \le \cU^{(\alpha)}(\psi_1, \psi_2)$ and $\cR_\infty(\psi_2) \le \cT^{(\alpha)}_\infty(\psi_2) \le \cU^{(\alpha)}_\infty(\psi_2)$.

\section{Asymptotic power laws and separations}\label{sec:power_law}

In this section, we evaluate the algebraic expressions derived in our main result (Theorem~\ref{thm:main_theorem}) as well as the quantities $\cU^{(\alpha)}$, $\cT^{(\alpha)}$, $\cA$, and $\cR$, before formally presenting the theorem. We examine their dependence with respect to the noise level $\tau^2$, the number of features $\psi_1 = \lim_{d \to \infty} N/d$, and the sample size $\psi_2 = \lim_{d \to \infty} n/d$, and we further infer their asymptotic power laws for large $\psi_1$ and $\psi_2$.

\subsection{Norm of the minimum norm interpolator}\label{sec:kernel_norm}

Since we are considering uniform convergence bounds over the norm ball of size $\alpha$ times $\cA_\infty(\psi_2)$ (the norm of the min-norm interpolator), let's first examine how $\cA_\infty(\psi_2)$ scale with $\psi_2$. 
As we shall see, $\cA_\infty(\psi_2)$ behaves differently in the noiseless ($\tau^2=0$) and noisy ($\tau^2>0$) settings, so here we explicitly mark the dependence on $\tau^2$, i.e. $\cA_\infty(\psi_2;\tau^2)$.

The inferred asymptotic power law gives (c.f. Figure \ref{fig:kernel_norm})
\[
\begin{aligned}
  \cA_\infty(\psi_2; \tau^2>0) \sim&~ \psi_2,\\
  \cA_\infty(\psi_2; \tau^2=0) \sim&~ 1,
\end{aligned}
\]
where $X_1(\psi) \sim X_2(\psi)$ for large $\psi$ means that 
\[
\lim_{\psi \to \infty} \log(X_1(\psi)) / \log(X_2(\psi)) = 1.
\]
In words, when there is no label noise ($\tau^2=0$), we can interpolate infinite data even with a finite norm. When the responses are noisy $(\tau^2>0)$, interpolation requires a large norm that is proportional to the number of samples.

On a high level, our statement echoes the finding of \citet{belkin2018understand}, where they study a binary classification problem using the kernel machine, and prove that an interpolating classifier requires RKHS norm to grow at least exponentially with $n^{1/d}$ for fixed dimension $d$. Here instead we consider the high dimensional setting and we show a linear grow in $\psi_2 = \lim_{d \to \infty} n/d$.

\subsection{Kernel regime with noiseless data}\label{sec:kernel_noiseless}
 We first look at the noiseless setting ($\tau^2=0$) and present the asymptotic power law for the uniform convergence $\cU_\infty^{(\alpha)}$ over the low-norm ball, the uniform convergence over interpolators $\cT_\infty^{(\alpha)}$ in the low norm ball, and the minimum norm risk  $\cR_\infty$ from \eqref{eqn:cU_inf_alpha} \eqref{eqn:cT_inf_alpha} \eqref{eqn:cR_inf}, respectively.

In this setting, the inferred asymptotic power law of $\cU_\infty^{(\alpha)}(\psi_2)$, $\cT_\infty^{(\alpha)}(\psi_2)$, and $\cR_\infty(\psi_2)$ gives (c.f. Figure \ref{fig:kernel_noiseless_utr})
\[
\begin{aligned}
\cU_\infty^{(\alpha)}(\psi_2;\tau^2=0) &\sim \psi_2^{-1/2},\\
\cT_\infty^{(\alpha)}(\psi_2;\tau^2=0) &\sim \psi_2^{-1},\\
\cR_\infty^{(\alpha)}(\psi_2;\tau^2=0) &\sim \psi_2^{-2}.
\end{aligned}
\]
As we can see, all the three quantities converge to $0$ in the large sample limit, which indicates that uniform convergence is able to explain generalization in this setting. yet uniform convergence bounds do not correctly capture the convergence rate (in terms of $\psi_2$) of the generalization error.

\subsection{Kernel regime with noisy data}\label{sec:kernel_noisy}
In the noisy setting (fix $\tau^2>0$), the Bayes risk (minimal possible risk) is $\tau^2$. We study the excess risk and the excess version of uniform convergence bounds by subtracting the Bayes risk $\tau^2$. The inferred asymptotic power law gives (c.f. Figure \ref{fig:kernel_noisy_utr})
\[
\begin{aligned}
\cU^{(\alpha)}_\infty(\psi_2;\tau^2) - \tau^2 &\sim \psi_2^{1/2},\\
\cT^{(\alpha)}_\infty (\psi_2;\tau^2) - \tau^2 &\sim 1,\\
\cR_\infty(\psi_2;\tau^2) - \tau^2 &\sim \psi_2^{-1}.
\end{aligned}
\]
In the presence of label noise, the excess risk $\cR_\infty- \tau^2$ vanishes in the large sample limit. In contrast, the classical uniform convergence $\cU_\infty$ becomes vacuous,  whereas the uniform convergence over interpolators $\cT_\infty$ converges to a constant, which gives a non-vacuous bound of $\cR_\infty$. 

The decay of the excess risk of minimum norm interpolators even in the presence of label noise is no longer a surprising phenomenon in high dimensions \cite{liang2019risk, ghorbani2019linearized, bartlett2020benign}. A simple explanation of this phenomenon is that the nonlinear part of the activation function $\sigma$ has an implicit regularization effect \cite{mm19}. 

The divergence of $\cU^{(\alpha)}_\infty$ in the presence of response noise is partly due to that $\cA_\infty(\psi_2)$ blows up linearly in $\psi_2$ (c.f. Section \ref{sec:kernel_norm}). In fact, we can develop a heuristic intuition that $\cU_\infty(A, \psi_2; \tau^2) \sim A / \psi_2^{1/2}$. 
Then the scaling  $\cU^{(\alpha)}_\infty(\psi_2; \tau^2 > 0) \sim \cA_\infty(\psi_2; \tau^2 > 0) / \psi_2^{1/2} \sim \psi_2^{1/2}$ can be explained away by the power law of $\cA_\infty(\psi_2; \tau^2 > 0) \sim \psi_2$. 
In other words, the complexity of the function space  of interpolators grows faster than the sample size $n$, which leads to the failure of uniform convergence in explaining generalization. This echoes the findings in \citet{zico2019unable}. 

To illustrate the scaling $\cU_\infty(A, \psi_2)\sim A/\psi_2^{1/2}$. We fix all other parameters ($\mu_1, \mu_\star, \tau, F_1$), and examine the dependence of $\cU_\infty$ on $A$ and $\psi_2$. We choose $A = A(\psi_2)$ according to different power laws $A(\psi_2) \sim \psi_2^p$ for $p = 0, 0.25, 0.5, 0.75, 1$. The inferred asymptotic power law gives $\cU_{\infty}(A(\psi_2), \psi_2)\sim \psi_2^{p-0.5}$ (c.f. Figure \ref{fig:uAfunc}). This provides an evidence for the relation $\cU_\infty(A,\psi_2) \sim A/ \psi_2^{1/2}$. 
\begin{figure}[ht]
  \begin{center}
    \includegraphics[width=0.35\textwidth]{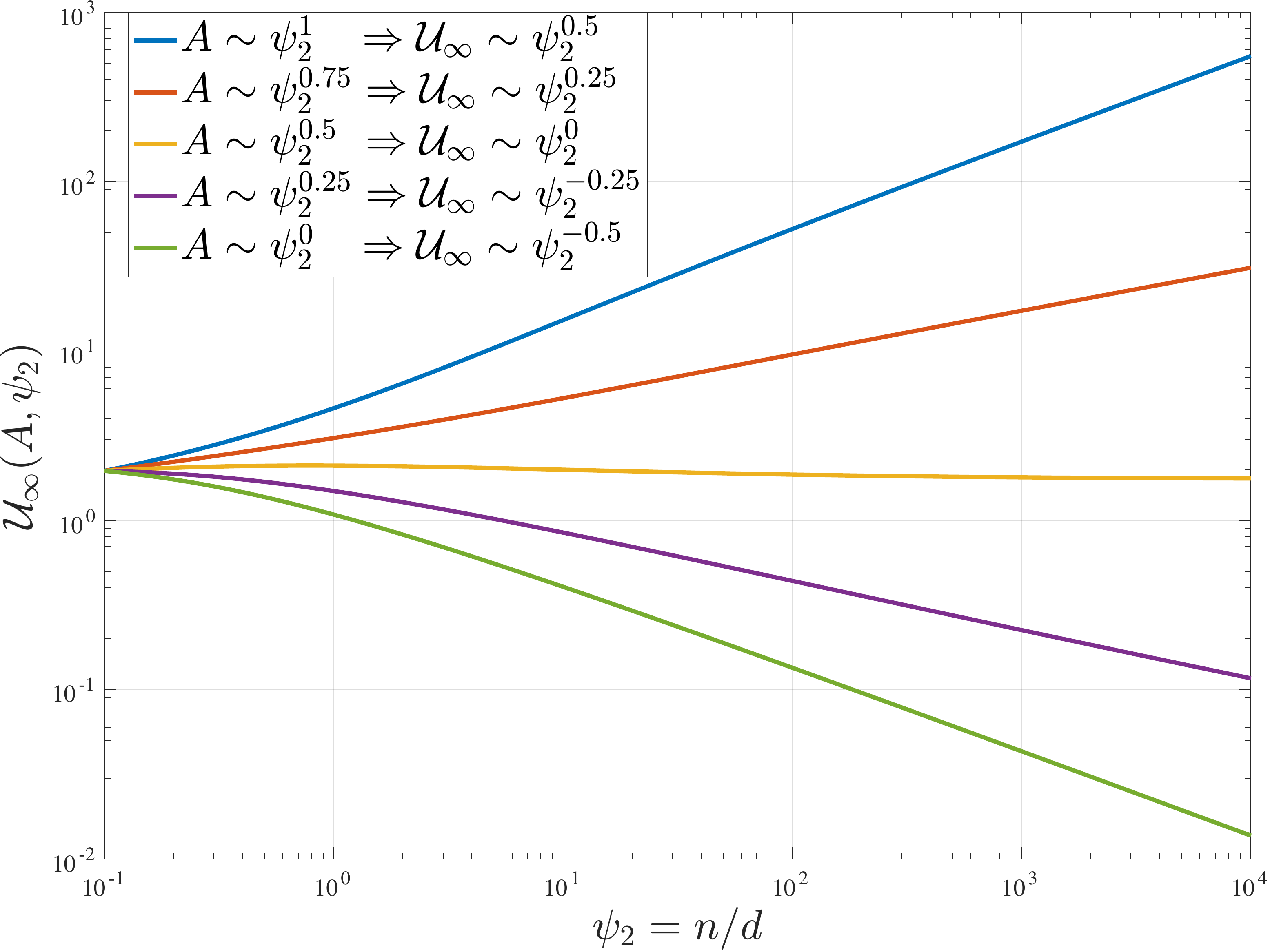}
    \vskip -0.1in
    \caption{Uniform convergence $\cU_\infty(A(\psi_2), \psi_2)$ over the norm ball in the kernel regime $\psi_1\rightarrow\infty$. The size of the norm ball $A = A(\psi_2)$ is chosen according to different power laws as shown in the legend.}
  \label{fig:uAfunc}
  \end{center}
  \vskip -0.2in
\end{figure}

\subsection{Finite-width regime}\label{sec:finite_width}

\begin{figure*}[ht]
  \begin{center}
  	\subfigure[]{
    \includegraphics[width=.33\textwidth]{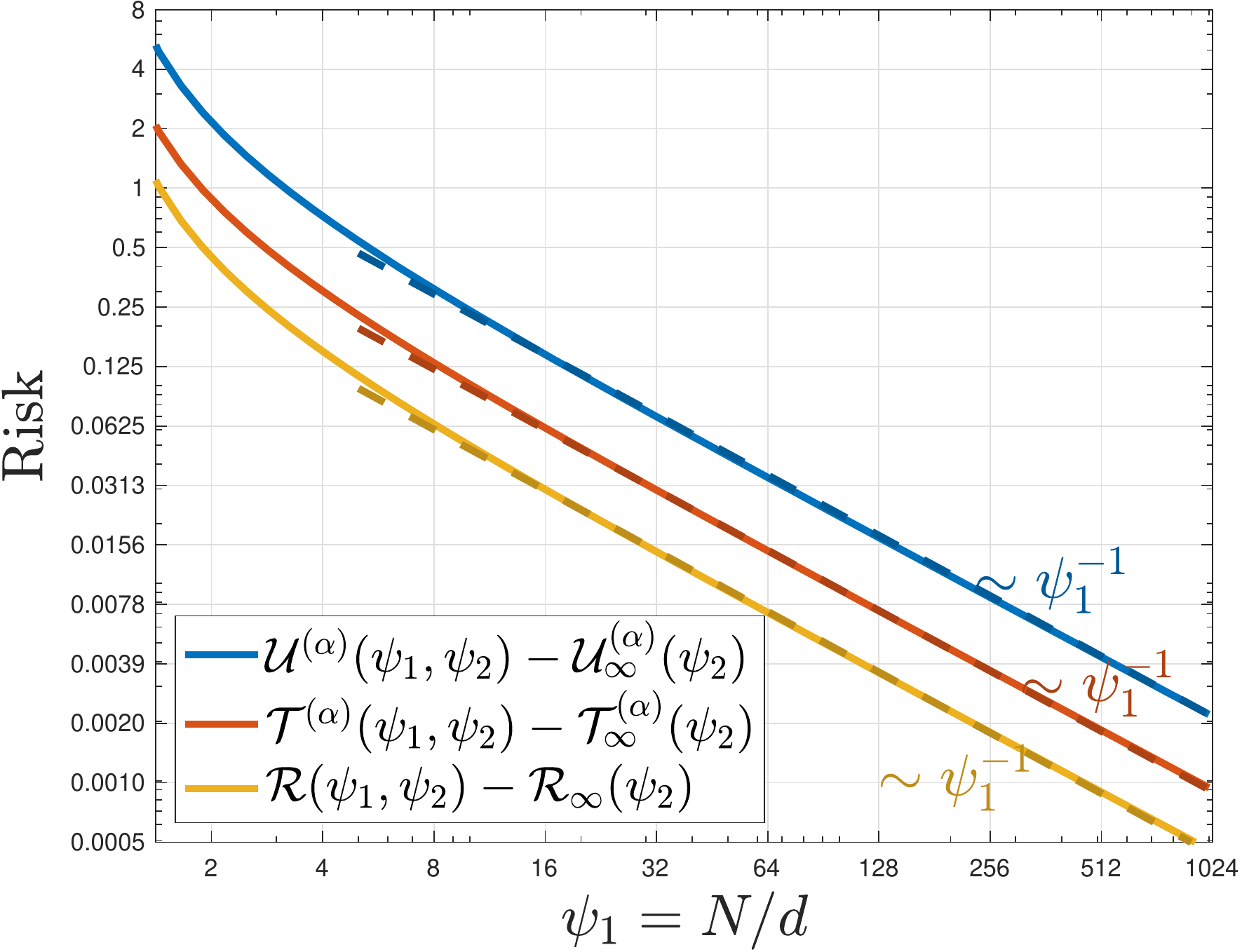}
    \label{fig:rf_utr}
    }
    \subfigure[]{
    \includegraphics[width=.33\textwidth]{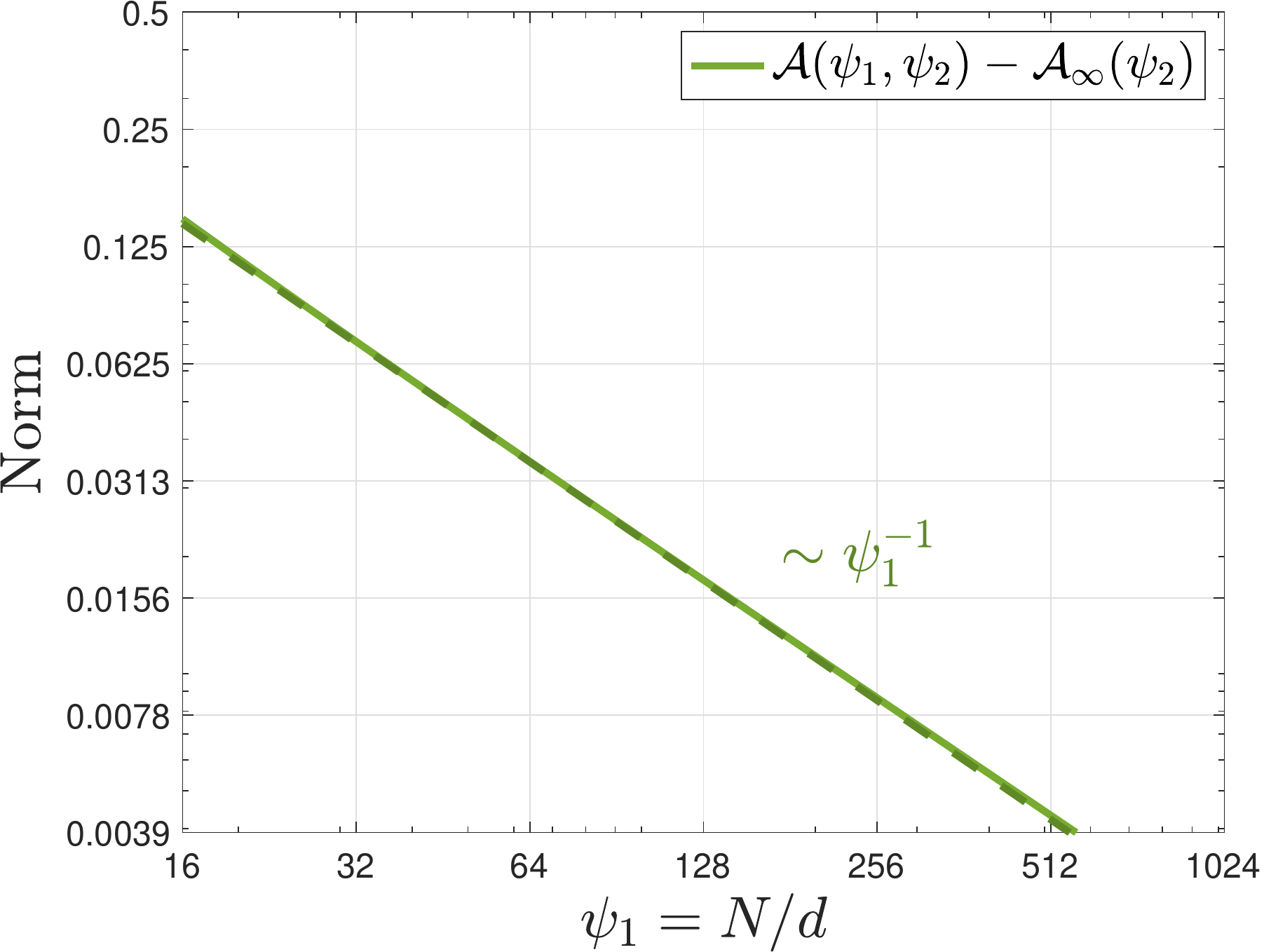}
    \label{fig:rf_norm}
    }
    \vskip -0.1in
    \caption{Random feature regression with the number of sample $\psi_2=1.5$, activation function $\sigma(x) = \max(0, x)- 1/\sqrt{2\pi}$, target function $f_d(\bx)=\<\bbeta, \bx\>$ with $\|\bbeta\|_2^2=1$, and noise level $\tau^2=0.1$. The horizontal axes are the number of features $\psi_1$. The solid lines are the the algebraic expressions derived in the main theorem (Theorem \ref{thm:main_theorem}). The dashed lines are the function $\psi_1^p$ in the log scale.
    Figure \ref{fig:rf_utr}: Comparison of the classical uniform convergence in the norm ball of size level $\alpha = 1.5$ (Eq. \eqref{eqn:cU_alpha}, blue curve), the uniform convergence over interpolators in the same norm ball (Eq. \eqref{eqn:cT_alpha}, red curve), the risk of minimum norm interpolator (Eq. \eqref{eqn:cR}, yellow curve).
    Figure \ref{fig:rf_norm}: Minimum norm required to interpolate the training data (Eq. \eqref{eqn:cA}).}
  \label{fig:rf_psi1}
  \end{center}
  \vskip -0.2in
\end{figure*}

Here we shift attention to the dependence of $\cU$, $\cT$, and $\cR$ on the number of features $\psi_1$.
We fix the number of training samples $\psi_2$, noise level $\tau^2 > 0$, and norm level $\alpha > 1$ similar as before. Since $\cU^{\alpha}\rightarrow\cU^{\alpha}_\infty, \cT^{\alpha}\rightarrow\cT^{\alpha}_\infty$ and $\cR\rightarrow\cR_\infty$ as $\psi_1\rightarrow\infty$, we look at the dependence of $\cU^{\alpha}-\cU^{\alpha}_\infty, \cT^{\alpha}-\cT^{\alpha}_\infty$ and $\cR^{\alpha}-\cR^{\alpha}_\infty$ with respect to $\psi_1$. The inferred asymptotic law gives (c.f. Figure \ref{fig:rf_psi1})
\[
\begin{aligned}
\cU^{(\alpha)}(\psi_1, \psi_2) - \cU^{(\alpha)}_\infty(\psi_2) &\sim \psi_1^{-1},\\
\cT^{(\alpha)}(\psi_1, \psi_2) - \cT^{(\alpha)}_\infty(\psi_2) &\sim \psi_1^{-1},\\
\cR(\psi_1, \psi_2) - \cR_\infty(\psi_2) &\sim \psi_1^{-1}, \\
\cA(\psi_1, \psi_2) - \cA_\infty(\psi_2) &\sim \psi_1^{-1}.
\end{aligned}
\]


Note that large $\psi_1$ should be interpreted as the model being heavily overparametrized (a large width network).
This asymptotic power law implies that, both uniform convergence bounds correctly predict the decay of the test error with the increase of the number of features.

\paragraph{Remark on power laws.} For the derivation of the power laws in this section, instead of working with the analytical formula, we adopt an empirical approach: we perform linear fits with the inferred slopes, upon the numerical evaluations (of these expressions defined in Definition \ref{def:formula_U_T}) in the log-log scale. However, these linear fits are for the analytical formulae and do not involve randomness, and thus reliably indicate the true decay rates.


%% file: sections/mainthm.tex
\section{Main theorem}

In this section, we state the main theorem that presents the asymptotic expressions for the uniform convergence bounds. We will start by stating a few assumptions, which fall into two categories: Assumption \ref{ass:linear_target}, \ref{ass:activation}, and \ref{ass:linear}, which specify the setup for the learning task; Assumption \ref{ass:overline_U_invertable} and \ref{ass:exchange_limit}, which are technical in nature. 

\subsection{Modeling assumptions}
The three assumptions in this subsection specify the target function, the activation function, and the limiting regime.
\begin{assumption}[Linear target function]\label{ass:linear_target} 
We assume that $f_d\in L^2(\S^{d-1}(\sqrt{d}))$ with $f_d(\bx) = \< \bbeta^{(d)}, \bx\>$, where $\bbeta^{(d)} \in \R^d$ and 
\[
\lim_{d \to \infty} \| \bbeta^{(d)} \|_2^2 = \normf_1^2.
\]
\end{assumption}

We remark here that, if we are satisfied with heuristic formulae instead of rigorous results, we are able to deal with non-linear target functions, where the additional nonlinear part is effectively increasing the noise level $\tau^2$. This intuition was first developed in \cite{mm19}. 

\begin{assumption}[Activation function]\label{ass:activation}
Let $\sigma \in C^2(\R)$ with $\vert \sigma(u) \vert, \vert\sigma'(u)\vert, \vert\sigma''(u)\vert \le c_0 e^{c_1 \vert u \vert}$ for some constant $c_0, c_1 < \infty$. Define 
\[
\ob_0 \equiv \E[\sigma(G)], ~ \ob_1 \equiv \E[G \sigma(G)], ~ \ob_\star^2 \equiv \E[\sigma(G)^2] - \ob_0^2 - \ob_1^2,
\]
where expectation is with respect to $G \sim \cN(0, 1)$. Assume $\ob_0 = 0$, $0 < \ob_1^2, \ob_\star^2 < \infty$.
\end{assumption}

The assumption that $\ob_0 = 0$ is not essential and can be relaxed with a certain amount of additional technical work. 

\begin{assumption}[Proportional limit]\label{ass:linear}
Let $N = N(d)$ and $n = n(d)$ be sequences indexed by $d$. We assume that the following limits exist in $(0, \infty)$: 
\[
\lim_{d \to \infty} N(d) / d = \psi_1, ~~~~~~~ \lim_{d \to \infty} n(d) / d = \psi_2. 
\]

\end{assumption}

\subsection{Technical assumptions}

We will make some assumptions upon the properties of some random matrices that appear in the proof. These assumptions are technical and we believe they can be proved under more natural assumptions. However, proving them requires substantial technical work, and we defer them to future work. We note here that these assumptions are often implicitly required in papers that present intuitions using heuristic derivations. Instead, we ensure the mathematical rigor by listing them. See Section \ref{sec:discuss} for more discussions upon these assumptions.

We begin by defining some random matrices which are the key quantities that are used in the proof of our main results.
\begin{definition}[Block matrix and log-determinant]\label{def:log_determinant_A}
Let $\bX = (\bx_1, \ldots, \bx_n)^\sT \in \R^{n \times d}$ and $\bTheta = (\btheta_1, \ldots, \btheta_N)^\sT \in \R^{N \times d}$, where $\bx_i, \btheta_a \sim_{iid} \text{\normalfont Unif}(\S^{d-1}(\sqrt{d}))$, as mentioned in Section \ref{sec:model_setup}. Define 
\begin{align}
  \bZ &= \frac{1}{\sqrt{d}}\sigma\left( \frac{\bX\bTheta^\sT}{\sqrt{d}} \right),~~ \bZ_1 = \frac{\mu_1}{d}\bX\bTheta^\sT, \nonumber \\
  \bQ &= \frac{\bTheta\bTheta^\sT}{d},~~~~~~~~~~~~~~~~~~~ \bH = \frac{\bX\bX^\sT}{d},~ 
  \label{eqn:def_Q_H_Z_Z1}
\end{align}
and for $\bq = (s_1, s_2, t_1, t_2, q) \in \R^5$, we define
\[
\bA(\bq) \equiv \begin{bmatrix}
s_1 \id_N + s_2 \bQ & \bZ^\sT + p \bZ_1^\sT \\
\bZ + p \bZ_1 & t_1 \id_n + t_2 \bH\\
\end{bmatrix}. 
\]
Finally, we define the log-deteminant of $\bA(\bq)$ by
\[
G_d(\xi; \bq) \equiv \frac{1}{d} \sum_{i = 1}^{N + n} \Log \lambda_i\Big(\bA(\bq) - \xi \id_{n + N}\Big).
\]
Here $\Log$ is the complex logarithm with branch cut on the negative real axis and $\{\lambda_i(\bA)\}_{i \in [n + N]}$ is the set of eigenvalues of $\bA$. 
\end{definition}

The following assumption states that for properly chosen $\lambda$, some specific random matrices are well-conditioned. As we will see in the next section, this ensures that the dual problems in Eq. (\ref{eqn:uniform_lag}) and (\ref{eqn:uniform_zeroloss_lag}) are bounded with high probability. 


\def\bT{{\boldsymbol T}}
\def\nullZ{{\rm null}}
\def\proj{{\mathsf P}}
\begin{assumption}[Invertability]\label{ass:overline_U_invertable}
Consider the asymptotic limit as specified in Assumption \ref{ass:linear} the activation function as in Assumption \ref{ass:activation}. We assume the following. 
\begin{itemize} 
\item Denote $\overline \bU(\lambda) = \ob_1^2 \bQ  + (\ob_\star^2 - \psi_1 \lambda ) \id_N - \psi_2^{-1}\bZ^\sT \bZ$. There exists $\eps > 0$ and $\lambdau =\lambdau(\psi_1, \psi_2, \ob_1^2, \ob_\star^2)$, such that for any fixed $\lambda \in (\lambdau, \infty) \equiv \Lambdau$, with high probability, we have
\[
\overline \bU(\lambda) \preceq - \eps \id_N. 
\]
\item Denote $\overline \bT(\lambda) = \proj_{\nullZ} [\ob_1^2 \bQ  + (\ob_\star^2 - \psi_1 \lambda ) \id_N ] \proj_{\nullZ}$ where $\proj_{\nullZ} = \id_N - \bZ^\dagger \bZ$. There exists $\eps > 0$ and $\lambdat =\lambdat(\psi_1, \psi_2, \ob_1^2, \ob_\star^2)$, such that for any fixed $\lambda \in (\lambdat, \infty) \equiv \Lambdat$, with high probability we have 
\[
\overline \bT(\lambda) \preceq - \eps \proj_{\nullZ},
\]
and $\bZ$ has full row rank with $\sigma_{\min}(\bZ) \ge \eps$ (which requires $\psi_1 > \psi_2$). 
\end{itemize}
\end{assumption}

The following assumption states that the order of limits and derivatives regarding $G_d$ can be exchanged. 
\begin{assumption}[Exchangeability of limits]\label{ass:exchange_limit}
We denote
\[
\begin{aligned}
\cSu =&~ \{ (\mu_\star^2 - \lambda \psi_1, \mu_1^2, \psi_2, 0,0; \psi_1, \psi_2): \lambda \in (\lambdau, \infty) \},\\
\cSt =&~ \{ (\mu_\star^2 - \lambda \psi_1, \mu_1^2, 0, 0,0; \psi_1, \psi_2): \lambda \in (\lambdat, \infty) \},
\end{aligned} 
\]
where $\lambdau$ and $\lambdat$ are given in Assumption \ref{ass:overline_U_invertable} and depend on $(\psi_1, \psi_2, \ob_1^2, \ob_\star^2)$. For any fixed $(\bq ; \bpsi) = (s_1, s_2, t_1, t_2, p; \psi_1, \psi_2) \in \cSu \cup \cSt$, in the asymptotic limit as in Assumption \ref{ass:linear}, for $k = 1, 2$, we have
\[
\begin{aligned}
\lim_{u \to 0_+} \lim_{d \to \infty} \E[ \nabla_\bq^k G_d(\bi u; \bq)] = \lim_{u \to 0_+ } \nabla_\bq^k \Big( \lim_{d \to \infty} \E[ G_d(\bi u; \bq)] \Big),\\
\end{aligned}
\]
and 
\[
\begin{aligned}
\Big \| \nabla_\bq^k  G_d(0; \bq) - \lim_{u \to 0+} \lim_{d \to \infty} \E[\nabla_\bq^k  G_d(\bi u; \bq)] \Big\| = o_{d, \P}(1),  
\end{aligned}
\]
where $o_{d, \P}(1)$ stands for convergence to $0$ in probability. 
\end{assumption}

\subsection{From constrained forms to Lagrangian forms}

Before we give the asymptotics of $U$ and $T$ as defined in Eq. $\eqref{eqn:uniform}$ and $\eqref{eqn:uniform_zeroloss}$, we first consider their dual forms which are more amenable in analysis. These are given by
\begin{align}
\overline U(\lambda, N, n, d) \equiv&~ \sup_{\ba}  \Big[ R(\ba) - \what R_n(\ba) - \psi_1 \lambda \| \ba \|_2^2 \Big], \label{eqn:uniform_lag}\\
\overline T(\lambda, N, n, d) \equiv&~ \sup_{\ba}\inf_{\bmu} \Big[ R(\ba) - \lambda\psi_1 \| \ba \|_2^2 \label{eqn:uniform_zeroloss_lag}\\
&~ + 2 \< \bmu, \bZ \ba - \by/\sqrt{d}\ \>\Big]. \nonumber
\end{align}
The proposition below shows that the strong duality holds upon the constrained forms and their dual forms. 
\begin{proposition}[Strong Duality]\label{prop:strong_duality}
For any $A > 0$, we have 
\[
\begin{aligned}
U(A, N, n, d) =&~ \inf_{\lambda \ge 0} \Big[ \overline U(\lambda, N, n, d) + \lambda A \Big]. \\
\end{aligned}
\]
Moreover, for any $A > \psi_1 \|\ba_{\min} \|_2^2$, we have  
\[
\begin{aligned}
T(A, N, n, d) =&~ \inf_{\lambda \ge 0}  \Big[ \overline T(\lambda, N, n, d) + \lambda A \Big]. \\
\end{aligned}
\]
\end{proposition}
The proof of Proposition \ref{prop:strong_duality} is based on a classical result which states that strongly duality holds for quadratic programs with single quadratic constraint (Appendix B.1 in \citet{boyd_vandenberghe_2004}).
\subsection{Expressions of $\cU$ and $\cT$}

Proposition \ref{prop:strong_duality} transforms our task from computing the asymptotics of $U$ and $T$ to that of $\overline U$ and $\overline T$. The latter is given by the following proposition. 

\begin{proposition}\label{prop:concentration_lag}
Let the target function $f_d$ satisfy Assumption \ref{ass:linear_target}, the activation function $\sigma$ satisfy Assumption \ref{ass:activation}, and $(N, n, d)$ satisfy Assumption \ref{ass:linear}. In addition, let Assumption \ref{ass:overline_U_invertable} and \ref{ass:exchange_limit} hold. Then for $\lambda \in \Lambdau$, with high probability the maximizer in Eq. (\ref{eqn:uniform_lag}) can be achieved at a unique point $\overline \ba_U(\lambda)$, and we have
\[
\begin{aligned}
 \overline U(\lambda, N, n, d) = \overline\cU(\lambda, \psi_1, \psi_2) + o_{d, \P}(1),\\
 \psi_1\|\overline\ba_U(\lambda)\|_2^2  = \cA_U(\lambda, \psi_1, \psi_2) + o_{d, \P}(1).\\
  \end{aligned}
\]
Moreover, for any $\lambda \in \Lambdat$, with high probability the maximizer in Eq. (\ref{eqn:uniform_zeroloss_lag}) can be achieved at a unique point $\overline \ba_T(\lambda)$, and we have
\[
\begin{aligned}
\overline T(\lambda, N, n, d) = \overline\cT(\lambda, \psi_1, \psi_2) + o_{d, \P}(1),\\
\psi_1\|\overline\ba_T(\lambda)\|_2^2 = \cA_T(\lambda, \psi_1, \psi_2) + o_{d, \P}(1) .\\
\end{aligned}
\]
The functions $\overline \cU, \overline \cT, \cA_U, \cA_T$ are given in Definition \ref{def:analytic_expression_overline} in Appendix \ref{sec:analytic_expression_overline}. 
\end{proposition}

\begin{remark}\label{rmk:heuristic_def}
\def\ext{{\rm ext}}
Here we present the heuristic formulae of $\overline \cU, \overline \cT, \cA_U, \cA_T$, and defer their rigorous definition to the appendix. Define a function $g_0(\bq; \bpsi)$ by
\begin{equation}\label{eqn:def_g0_heuristic}
\begin{aligned}
&g_0(\bq; \bpsi) \equiv~ \ext_{z_1, z_2} \Big[ \log\big( (s_2 z_1 + 1)(t_2 z_2 + 1) \\
&- \ob_1^2 (1 + p)^2 z_1 z_2 \big) - \ob_\star^2 z_1 z_2 + s_1 z_1 +  t_1 z_2 \\
&- \psi_1 \log (z_1 / \psi_1) - \psi_2 \log (z_2 / \psi_2) - \psi_1 - \psi_2 \Big], 
\end{aligned}
\end{equation}
where $\ext$ stands for setting $z_1$ and $z_2$ to be stationery (which is a common symbol in statistical physics heuristics). We then take
\[
\begin{aligned}
 \overline \cU(\lambda, \bpsi) = \normf_1^2 ( 1 -  \mu_1^2  \gamma_{s_2} -  \gamma_{p} - \gamma_{t_2} ) + \tau^2( 1 - \gamma_{t_1}),  \\
\end{aligned}
\]
where $\gamma_a \equiv \partial_a g_0(\bq; \bpsi)  \vert_{\bq = (\mu_\star^2 - \lambda\psi_1, \mu_1^2, \psi_2, 0,0)}$ for the symbol $a \in \{s_1, s_2, t_1, t_2, p \}$, and 
\[
\begin{aligned}
\overline \cT(\lambda, \bpsi) = \normf_1^2 ( 1 -  \mu_1^2  \nu_{s_2} -  \nu_{p} - \nu_{t_2} ) + \tau^2( 1 - \nu_{t_1}), 
\end{aligned}
\]
where we define $\nu_a \equiv \partial_a g_0(\bq; \bpsi) \vert_{\bq = (\mu_\star^2 - \lambda\psi_1, \mu_1^2, 0, 0,0)}$ for symbols $a \in \{s_1, s_2, t_1, t_2, p \}$. 
Finally $\cA_U = - \partial_\lambda \overline \cU$, $\cA_T = - \partial_\lambda \overline \cT$. By a further simplification, we can express these formulae to be rational functions of $(\ob_1^2, \ob_\star^2, \lambda, \psi_1, \psi_2, m_1, m_2)$ where $(m_1, m_2)$ is the stationery point of the variational problem in Eq. (\ref{eqn:def_g0_heuristic})  (c.f. Remark \ref{rmk:simplification}). 
\end{remark}

We next define $\cU$ and $\cT$ to be dual forms of $\overline \cU$ and $\overline \cT$. 

\begin{definition}[Formula for uniform convergence bounds]\label{def:formula_U_T} 
For $A \in \Gamma_U \equiv \{\cA_U(\lambda, \psi_1, \psi_2): \lambda \in \Lambdau \}$, define
\[
\begin{aligned}
\cU(A, \psi_1, \psi_2) \equiv&~ \inf_{\lambda \ge 0} \Big[ \overline \cU(\lambda, \psi_1, \psi_2) + \lambda A \Big]. \\
\end{aligned}
\]
For $A \in \Gamma_T \equiv \{\cA_T(\lambda, \psi_1, \psi_2): \lambda \in \Lambdat \}$, define
\[
\begin{aligned}
\cT(A, \psi_1, \psi_2) \equiv&~  \inf_{\lambda \ge 0} \Big[ \overline \cT(\lambda, \psi_1, \psi_2)+ \lambda A \Big].\\
\end{aligned}
\]
\end{definition}

Finally, we are ready to present the main theorem of this paper, which states that the uniform convergence bounds $U(A, N, n, d)$ and $T(A, N, n, d)$ converge to the formula presented in the definition above.
\begin{theorem}\label{thm:main_theorem}
Let the same assumptions in Proposition \ref{prop:concentration_lag} hold. For any $A \in \Gamma_U$, we have 
\begin{align}\label{eqn:main_U}
U(A, N, n, d) = \cU(A, \psi_1, \psi_2) + o_{d, \P}(1), \
\end{align}
and for $A \in \Gamma_T$ we have 
 \begin{align}\label{eqn:main_T}
T(A, N, n, d) = \cT(A, \psi_1, \psi_2) + o_{d, \P}(1), 
\end{align}
where functions $\cU$ and $\cT$ are given in Definition \ref{def:formula_U_T}. 
\end{theorem}
The proof of this theorem is contained in Section \ref{sec:proof_main_thm}. 

%% file: sections/discussions.tex
\vspace{-0.5em}
\section{Discussions}\label{sec:discuss}
\vspace{-0.5em}

In this paper, we calculated the uniform convergence bounds for random features models in the proportional scaling regime. Our results exhibit a setting in which standard uniform convergence bound is vacuous while uniform convergence over interpolators gives a non-trivial bound of the actual generalization error. 


\vspace{-0.8em}
\paragraph{Modeling assumptions and technical assumptions.}
We made a few assumptions to prove the main result Theorem \ref{thm:main_theorem}. Some of these assumptions can be relaxed. Indeed, if we assume a non-linear target function $f_d$ instead of a linear one as in Assumption \ref{ass:linear_target}, the non-linear part will behave like additional noises in the proportional scaling limit. However, proving this rigorously requires substantial technical work. Similar issue exists in \citet{mm19}. Moreover, it is not essential to assume vanishing $\ob_0^2$ in Assumption \ref{ass:activation}. 

Assumption \ref{ass:overline_U_invertable} and \ref{ass:exchange_limit} involve some properties of specific random matrices.
We believe these assumptions can be proved under more natural assumptions on the activation function $\sigma$. However, proving these assumptions requires developing some sophisticated random matrix theory results, which could be of independent interest. 

\vspace{-0.8em}
\paragraph{Relationship with non-asymptotic results.}
 We hold the same opinion as in \citet{abbaras2020rademacher}: the exact formulae in the asymptotic limit can provide a complementary view to the classical theories of generalization.  
 On the one hand, asymptotic formulae can be used to quantify the tightness of non-asymptotic bounds; on the other hand, the asymptotic formulae in many cases are comparable to non-asymptotic bounds. 
 For example, Lemma 22 in \citet{bartlettAndMendelson} coupled with the bound of Lipschitz constant of the square loss in proper regime implies that $\cU_\infty(A, \psi_2)$ have a non-asymptotic bound that scales linearly in $A$ and inverse proportional to $\psi_2^{1/2}$ (c.f. Proposition 6 of \citet{weinan2020machine}). This coincides with the intuitions in Section \ref{sec:kernel_noisy}.
 

\vspace{-0.8em}
\paragraph{Uniform convergence in other settings.}
A natural question is whether the power law derived in Section \ref{sec:power_law} holds for models in more general settings. 
One can perform a similar analysis to calculate the uniform convergence bounds in a few other settings \cite{montanari2019generalization,dhifallah2020precise,hu2020universality}.
We believe the power law may be different, but the qualitative properties of uniform convergence bounds will share some similar features. 

\vspace{-0.8em}
\paragraph{Relationship with~\citet{zhou2021uniform}.}
The separation of uniform convergence bounds ($U$ and $T$) is first pointed out by \citet{zhou2021uniform}, where the authors worked with the linear regression model in the ``junk features" setting. We believe random features model are more natural models to illustrate the separation: in \citet{zhou2021uniform}, there are some unnatural parameters $\lambda_n, d_J$ that are hard to make connections to deep learning models, while the random features model is closely related to two-layer neural networks. 


%% file: sections/additional_material.tex

\section{Definitions of quantities in the main text}\label{sec:analytic_expression_overline}

\subsection{Full definitions of $\overline \cU$, $\overline \cT$, $\cA_U$, and $\cA_T$ in Proposition \ref{prop:concentration_lag}}\label{sec:analytic_expression_U_T}

We first define functions $m_1(\cdot), m_2(\cdot)$, which could be understood as the limiting partial Stieltjes transforms of $\bA(\bq)$ (c.f. Definition \ref{def:log_determinant_A}). 
 \begin{definition}[Limiting partial Stieltjes transforms]\label{def:Stieltjes}
For $\xi \in \C_+$ and $\bq \in \cQ$ where
 \begin{equation}\label{eqn:definition_of_cQ}
 \cQ = \{ (s_1, s_2, t_1, t_2, p): \vert s_2 t_2 \vert \le \ob_1^2(1 + p)^2 / 2 \},  
 \end{equation}
 define functions $\sFone(\,\cdot\, ,\,\cdot\,;\xi;\bq, \psi_1,\psi_2, \ob_1, \ob_\star), \sFtwo(\,\cdot\, ,\,\cdot\,;\xi;\bq, \psi_1,\psi_2, \ob_1, \ob_\star):\complex\times\complex \to \complex$ via:
 \[
 \begin{aligned}
 &\sFone(m_1,m_2;\xi;\bq, \psi_1,\psi_2, \ob_1, \ob_\star) \equiv \psi_1\Big(-\xi+s_1 - \ob_\star^2 m_2 +\frac{(1+t_2m_2)s_2- \ob_1^2 (1 + p)^2 m_2}{(1+s_2m_1)(1+t_2m_2)- \ob_1^2 (1 + p)^2 m_1m_2}\Big)^{-1}\,, \\
 &\sFtwo(m_1,m_2;\xi;\bq, \psi_1,\psi_2, \ob_1, \ob_\star) \equiv \psi_2\Big(-\xi+t_1 - \ob_\star^2 m_1 +\frac{(1+s_2 m_1)t_2- \ob_1^2 (1 + p)^2 m_1}{(1+t_2m_2)(1+s_2m_1)- \ob_1^2 (1 + p)^2 m_1m_2}\Big)^{-1}\, .
 \end{aligned}
 \]
 Let $m_1(\,\cdot\, ;\bq; \bpsi)$ $m_2(\,\cdot\, ;\bq; \bpsi):\complex_+\to\complex_+$ be defined, for $\Im(\xi)\ge C$ a sufficiently large constant, as the unique solution of 
 the equations
 \begin{equation}\label{eq:FixedPoint}
 \begin{aligned}
 m_{1} &= \sFone(m_1,m_2;\xi;\bq, \psi_1,\psi_2, \ob_1, \ob_\star),\\
 m_{2} &= \sFtwo(m_1, m_2;\xi;\bq,\psi_1,\psi_2, \ob_1, \ob_\star)\,
 \end{aligned}
 \end{equation}
 subject to the condition $\vert m_1\vert \le \psi_1/\Im(\xi)$, $\vert m_2\vert \le \psi_2/\Im(\xi)$. Extend this definition to $\Im(\xi) >0$ by requiring $m_1,m_2$ to be analytic functions in $\complex_+$. 
 \end{definition}
 
We next define the function $g(\cdot)$ that will be shown to be the limiting log determinant of $\bA(\bq)$. 
\begin{definition}[Limiting log determinants]\label{def:limiting_log_determinant}
For $\bq = (s_1, s_2, t_1, t_2, p)$ and $\bpsi = (\psi_1, \psi_2)$, define
\begin{equation}\label{eqn:log_determinant_variation}
\begin{aligned}
\Xi(\xi, z_1, z_2; \bq; \bpsi) \equiv&~ \log[(s_2 z_1 + 1)(t_2 z_2 + 1) - \ob_1^2 (1 + p)^2 z_1 z_2] - \ob_\star^2 z_1 z_2 \\
 &+ s_1 z_1 +  t_1 z_2  - \psi_1 \log (z_1 / \psi_1) - \psi_2 \log (z_2 / \psi_2)  - \xi (z_1 + z_2) - \psi_1 - \psi_2.
\end{aligned}
\end{equation}
Let $m_1(\xi; \bq; \bpsi), m_2(\xi; \bq; \bpsi)$ be defined as the analytic continuation of solution of Eq. (\ref{eq:FixedPoint}) as defined in Definition \ref{def:Stieltjes}. Define
\begin{equation}\label{eqn:formula_g_mm19}
g(\xi; \bq; \bpsi) = \Xi(\xi, m_1(\xi; \bq; \bpsi), m_2(\xi; \bq; \bpsi); \bq; \bpsi). 
\end{equation}
\end{definition}

We next give the definitions of $\overline \cU$, $\overline \cT$, $\cA_U$, and $\cA_T$. 

\begin{definition}[$\overline \cU$, $\overline \cT$, $\cA_U$, and $\cA_T$ in Proposition \ref{prop:concentration_lag}]\label{def:analytic_expression_overline}
For any $\lambda \in \Lambdau$, define
\[
\begin{aligned}
 \cA_U(\lambda, \psi_1, \psi_2) &= -\lim_{u\rightarrow 0_+}\left[ \psi_1\left( \normf_1^2\mu_1^2\partial_{s_1s_2} +\normf_1^2\partial_{s_1p} + \normf_1^2\partial_{s_1t_2} + \tau^2 \partial_{s_1 t_1}\right) g(\bi u; \bq; \bpsi)\Big\vert_{\bq = \bq_U}\right],\\
 \overline \cU(\lambda, \psi_1, \psi_2) &= F_1^2 + \tau^2 - \lim_{u \to 0_+} \left[ \big( \normf_1^2\mu_1^2  \partial_{s_2} +  \normf_1^2\partial_{p} + \normf_1^2\partial_{t_2}  + \tau^2 \partial_{t_1}  \big)g(\bi u; \bq; \bpsi) \Big\vert_{\bq = \bq_U}\right], \\
 \cA_T(\lambda, \psi_1, \psi_2) &= -\lim_{u\rightarrow 0_+}\left[ \psi_1\left( \normf_1^2\mu_1^2\partial_{s_1s_2} +\normf_1^2\partial_{s_1p} + \normf_1^2\partial_{s_1t_2} + \tau^2 \partial_{s_1 t_1}\right) g(\bi u; \bq; \bpsi)\Big\vert_{\bq = \bq_T}\right],\\
 \overline \cT(\lambda, \psi_1, \psi_2) &= F_1^2 + \tau^2 - \lim_{u \to 0_+} \left[ \big( \normf_1^2\mu_1^2  \partial_{s_2} +  \normf_1^2\partial_{p} + \normf_1^2\partial_{t_2}  + \tau^2 \partial_{t_1}  \big)g(\bi u; \bq; \bpsi) \Big\vert_{\bq = \bq_T}\right],\\
\end{aligned}
\]
where $\bq_U = (\mu_\star^2 - \lambda\psi_1, \mu_1^2, \psi_2, 0,0), \bq_T = (\mu_\star^2 - \lambda\psi_1, \mu_1^2, 0, 0,0)$.
\end{definition}

In the following, we give a simplified expression for $\overline \cU$ and $\cA_U$. 

\begin{remark}[Simplification of $\overline \cU$ and $\cA_U$]\label{rmk:simplification}
Define $\zeta, \overline\lambda$ as the rescaled version of $\mu_1^2$ and $\lambda$ 
\[\ratio = \frac{\mu_1^2}{\mu_\star^2},~~ \overline\lambda=\frac{\lambda}{\mu_\star^2}.\]
Let $m_1(\,\cdot\,; \bpsi)$ $m_2(\,\cdot\,; \bpsi):\complex_+\to\complex_+$ be defined, for $\Im(\xi)\ge C$ a sufficiently large constant, as the unique solution of 
 the equations
 \begin{equation}
 \begin{aligned}
 m_{1} &= \psi_1\left[-\xi+(1-\overline\lambda\psi_1) - m_2 + \frac{\ratio(1-m_2)}{1+\ratio m_1-\ratio m_1 m_2}\right]^{-1},\\
 m_{2} &= -\psi_2\left[\xi+\psi_2-m_1-\frac{\ratio m_1}{1+\ratio m_1 - \ratio m_1 m_2}\right]^{-1},
 \end{aligned}
 \end{equation}
 subject to the condition $\vert m_1\vert \le \psi_1/\Im(\xi)$, $\vert m_2\vert \le \psi_2/\Im(\xi)$. Extend this definition to $\Im(\xi) >0$ by requiring $m_1,m_2$ to be analytic functions in $\complex_+$. Let
\[
\begin{aligned}
\overline{m}_1 = \lim_{u\rightarrow\infty}m_1(\bi u, \bpsi),\\
\overline{m}_2 = \lim_{u\rightarrow\infty}m_2(\bi u, \bpsi).
\end{aligned}
\]
Define 
\[
\begin{aligned}
	\chi_1 &= \overline{m}_1 \ratio-\overline{m}_1 \overline{m}_2 \ratio+1,\\
	\chi_2 &= \overline{m}_1 -\psi_2+\frac{\overline{m}_1 \ratio}{\chi_1},\\
	\chi_3 &= \overline{\lambda}  \psi_1+\overline{m}_2-1+\frac{\ratio \left(\overline{m}_2-1\right)}{\chi_1}.
\end{aligned}
\]
Define two polynomials $\cE_1, \cE_2$ as 
\[
\begin{aligned}
	\cE_1(\psi_1, \psi_2, \overline{\lambda}, \ratio) =&~ \psi_1^2(\psi_2 \chi_1^4+\psi_2 \chi_1^2 \ratio),\\
	\cE_2(\psi_1, \psi_2, \overline{\lambda}, \ratio) =&~ \psi_1^2 (\chi_1^2 \chi_2^2 \overline{m}_2^2 \ratio-2 \chi_1^2 \chi_2^2 \overline{m}_2 \ratio+\chi_1^2 \chi_2^2 \ratio +\psi_2 \chi_1^2-\psi_2 \overline{m}_1^2 \overline{m}_2^2 \ratio^3+2 \psi_2 \overline{m}_1^2 \overline{m}_2 \ratio^3-\psi_2 \overline{m}_1^2 \ratio^3+\psi_2 \ratio),\\
	\cE_3(\psi_1, \psi_2, \overline{\lambda}, \ratio) =&~ -\chi_1^4 \chi_2^2 \chi_3^2+\psi_1 \psi_2 \chi_1^4 +\psi_1 \chi_1^2 \chi_2^2 \overline{m}_2^2 \ratio^2-2 \psi_1 \chi_1^2 \chi_2^2 \overline{m}_2 \ratio^2+\psi_1 \chi_1^2 \chi_2^2 \ratio^2\\
	&~+\psi_2 \chi_1^2 \chi_3^2 \overline{m}_1^2 \ratio^2+2 \psi_1 \psi_2 \chi_1^2 \ratio -\psi_1 \psi_2 \overline{m}_1^2 \overline{m}_2^2 \ratio^4+2 \psi_1 \psi_2 \overline{m}_1^2 \overline{m}_2 \ratio^4-\psi_1 \psi_2 \overline{m}_1^2 \ratio^4+\psi_1 \psi_2 \ratio^2.
\end{aligned}
\]
Then
\[
\begin{aligned}
	\overline\cU(\overline{\lambda}, \psi_1, \psi_2) &= -\frac{\left(\overline{m}_2-1\right) \left(\tau ^2\chi_1(\psi_1, \psi_2, \overline{\lambda}, \ratio)+F_1^2\right)}{\chi_1(\psi_1, \psi_2, \overline{\lambda}, \ratio)},\\
	\cA_U(\overline{\lambda}, \psi_1, \psi_2) &= \frac{\tau^2\cE_1(\psi_1, \psi_2, \overline{\lambda}, \ratio)+F_1^2\cE_1(\psi_1, \psi_2, \overline{\lambda}, \ratio)}{\cE_2(\psi_1, \psi_2, \overline{\lambda}, \ratio)}.
\end{aligned}
\]
\end{remark}

\begin{remark}[Simplification of $\overline \cT$ and $\cA_T$]\label{rmk:simplification_interpolating}
Define $\zeta, \overline\lambda$ as the rescaled version of $\mu_1^2$ and $\lambda$ 
\[\ratio = \frac{\mu_1^2}{\mu_\star^2},~~ \overline\lambda=\frac{\lambda}{\mu_\star^2}.\]
Let $m_1(\,\cdot\,; \bpsi)$ $m_2(\,\cdot\,; \bpsi):\complex_+\to\complex_+$ be defined, for $\Im(\xi)\ge C$ a sufficiently large constant, as the unique solution of 
 the equations
 \begin{equation}
 \begin{aligned}
 m_{1} &= \psi_1\left[-\xi+(1-\overline\lambda\psi_1) - m_2 + \frac{\ratio(1-m_2)}{1+\ratio m_1-\ratio m_1 m_2}\right]^{-1},\\
 m_{2} &= -\psi_2\left[\xi+m_1+\frac{\ratio m_1}{1+\ratio m_1 - \ratio m_1 m_2}\right]^{-1},
 \end{aligned}
 \end{equation}
 subject to the condition $\vert m_1\vert \le \psi_1/\Im(\xi)$, $\vert m_2\vert \le \psi_2/\Im(\xi)$. Extend this definition to $\Im(\xi) >0$ by requiring $m_1,m_2$ to be analytic functions in $\complex_+$. Let
\[
\begin{aligned}
\overline{m}_1 = \lim_{u\rightarrow\infty}m_1(\bi u, \bpsi),\\
\overline{m}_2 = \lim_{u\rightarrow\infty}m_2(\bi u, \bpsi).
\end{aligned}
\]
Define
\[
\begin{aligned}
	\chi_4 = \overline{m}_1 + \frac{\overline{m}_1\ratio}{\chi_1(\overline{m}_1, \overline{m}_2, \ratio)},
\end{aligned}
\]
and 
\[
\begin{aligned}
	\chi_1 &= \overline{m}_1 \ratio-\overline{m}_1 \overline{m}_2 \ratio+1,\\
	\chi_3 &= \overline{\lambda}  \psi_1+\overline{m}_2-1+\frac{\ratio \left(\overline{m}_2-1\right)}{\chi_1},
\end{aligned}
\]
where the definitions of $\chi_1, \chi_3$ are the same as in Remark \ref{rmk:simplification}. 
Define three polynomials $\cE_3, \cE_4, \cE_5$ as 
\[
\begin{aligned}
\cE_4(\psi_1, \psi_2, \overline{\lambda}, \ratio) =& \psi_1 \Big(\psi_2 \chi_1^4 \chi_4^3+\chi_1^4 \chi_4^2 \overline{m}_1^3 \overline{m}_2^2 \ratio^3-2 \chi_1^4 \chi_4^2 \overline{m}_1^3 \overline{m}_2 \ratio^3+\chi_1^4 \chi_4^2 \overline{m}_1^3 \ratio^3+2 \chi_1^3 \chi_4^2 \overline{m}_1^3 \overline{m}_2^2 \ratio^2\\
&-4 \chi_1^3 \chi_4^2 \overline{m}_1^3 \overline{m}_2 \ratio^2+2 \chi_1^3 \chi_4^2 \overline{m}_1^3 \ratio^2-\psi_2 \chi_1^3 \chi_4^2 \overline{m}_1 \ratio+\chi_1^2 \chi_4^2 \overline{m}_1^3 \overline{m}_2^2 \ratio-2 \chi_1^2 \chi_4^2 \overline{m}_1^3 \overline{m}_2 \ratio\\
&+\chi_1^2 \chi_4^2 \overline{m}_1^3 \ratio +\psi_2 \chi_1^2 \chi_4^2 \overline{m}_1 \ratio-\psi_2 \chi_1^2 \overline{m}_1^5 \overline{m}_2^2 \ratio^5+2 \psi_2 \chi_1^2 \overline{m}_1^5 \overline{m}_2 \ratio^5-\psi_2 \chi_1^2 \overline{m}_1^5 \ratio^5\\
&-2 \psi_2 \chi_1 \overline{m}_1^5 \overline{m}_2^2 \ratio^4+4 \psi_2 \chi_1 \overline{m}_1^5 \overline{m}_2 \ratio^4 -2 \psi_2 \chi_1 \overline{m}_1^5 \ratio^4-\psi_2 \overline{m}_1^5 \overline{m}_2^2 \ratio^3 \\
&+2 \psi_2 \overline{m}_1^5 \overline{m}_2 \ratio^3 -\psi_2 \overline{m}_1^5 \ratio^3 \Big),\\
\cE_5 (\psi_1, \psi_2, \overline{\lambda}, \ratio) =& \overline{m}_1 {\Big(\ratio+1+\overline{m}_1 \ratio -\overline{m}_1 \overline{m}_2 \ratio \Big)}^2 \Big(-\chi_1^4 \chi_3^2 \chi_4^2 \overline{m}_1^2 \\
&+\psi_1 \psi_2 \chi_1^4 \chi_4^2-2 \psi_1 \psi_2 \chi_1^3 \chi _{4} \overline{m}_1 \ratio+\psi_2 \chi_1^2 \chi_3^2 \overline{m}_1^4 \ratio^2 +\psi_1 \chi_1^2 \chi_4^2 \overline{m}_1^2 \overline{m}_2^2 \ratio^2 \\
& -2 \psi_1 \chi_1^2 \chi_4^2 \overline{m}_1^2 \overline{m}_2 \ratio^2+\psi_1 \chi_1^2 \chi_4^2 \overline{m}_1^2 \ratio^2+2 \psi_1 \psi_2 \chi_1^2 \chi _{4} \overline{m}_1 \ratio+\psi_1 \psi_2 \chi_1^2 \overline{m}_1^2 \ratio^2 \\
& -2 \psi_1 \psi_2 \chi_1 \overline{m}_1^2 \ratio^2-\psi_1 \psi_2 \overline{m}_1^4 \overline{m}_2^2 \ratio^4+2 \psi_1 \psi_2 \overline{m}_1^4 \overline{m}_2 \ratio^4-\psi_1 \psi_2 \overline{m}_1^4 \ratio^4+\psi_1 \psi_2 \overline{m}_1^2 \ratio^2\Big),\\
\cE_6(\psi_1, \psi_2, \overline{\lambda}, \ratio) =& \chi_1^2 \chi_4^2 \psi_1 \psi_2   \Big(\chi _{4} \chi_1^2-\overline{m}_1 \chi_1 \ratio+\overline{m}_1 \ratio\Big) {\Big(\overline{m}_1 \ratio-\overline{m}_1 \overline{m}_2 \ratio+1\Big)}^2.~~~~~~~~~~~~~~~~~~~~~~~~~~~~~~~~~~~~~~~~~~~~~~~~~
\end{aligned}
\]
Then
\[
\begin{aligned}
	\overline\cT(\overline{\lambda}, \psi_1, \psi_2) &= -\frac{\left(\overline{m}_2-1\right) \left(\tau ^2\chi_1(\psi_1, \psi_2, \overline{\lambda}, \ratio)+F_1^2\right)}{\chi_1(\psi_1, \psi_2, \overline{\lambda}, \ratio)},\\
	\cA_T(\overline{\lambda}, \psi_1, \psi_2) &= -\psi_1 \frac{F_1^2\cE_4(\psi_1, \psi_2, \overline{\lambda}, \ratio)+\tau^2\cE_6(\psi_1, \psi_2, \overline{\lambda}, \ratio)}{\cE_5(\psi_1, \psi_2, \overline{\lambda}, \ratio)}. 
\end{aligned}
\]
\end{remark}

\subsection{Definitions of $\cR$ and $\cA$} \label{sec:analytic_expression_R_A}
In this section, we present the expression of $\cR$ and $\cA$ from \citet{mm19} which are used in our results and plots.

\begin{definition}[Formula for the prediction error of minimum norm interpolator]
Define 
\[
\ratio = \mu_1^2/\mu_\star^2,~~ \rho = F_1^2/\tau^2
\]
Let the functions $\nu_1, \nu_2: \C_+ \to \C_+$ be be uniquely defined by the following conditions: $(i)$ $\nu_1$, $\nu_2$ are analytic on $\C_+$;
$(ii)$ For $\Im(\xi)>0$, $\nu_1(\xi)$, $\nu_2(\xi)$ satisfy the following equations
\begin{equation}
\begin{aligned}
\nu_1 =&~ \psi_1\Big(-\xi -  \nu_2 - \frac{\ratio^2 \nu_2}{1- \ratio^2 \nu_1\nu_2}\Big)^{-1}\, ,\\
\nu_2 =&~ \psi_2\Big(-\xi - \nu_1 - \frac{\ratio^2 \nu_1}{1- \ratio^2 \nu_1\nu_2}\Big)^{-1}\, ;
\end{aligned}
\end{equation}
$(iii)$  $(\nu_1(\xi), \nu_2(\xi))$ is the unique solution of these equations with $\vert \nu_1(\xi)\vert\le \psi_1/\Im(\xi)$, $\vert \nu_2(\xi) \vert \le \psi_2/\Im(\xi)$ for $\Im(\xi) > C$, with $C$ a sufficiently large constant.

Let  
\begin{equation}\label{eqn:definition_chi_main_formula}
\chi \equiv \lim_{u\rightarrow 0} \nu_1(\bi u) \cdot \nu_2(\bi u),
\end{equation}
and
\begin{equation}
\begin{aligned}
E_0(\ratio, \psi_1, \psi_2) \equiv&~  - \chi^5\ratio^6 + 3\chi^4 \ratio^4+ (\psi_1\psi_2 - \psi_2 - \psi_1 + 1)\chi^3\ratio^6 - 2\chi^3\ratio^4 - 3\chi^3\ratio^2 \\
&+ (\psi_1 + \psi_2 - 3\psi_1\psi_2 + 1)\chi^2\ratio^4 + 2\chi^2\ratio^2+ \chi^2+ 3\psi_1\psi_2\chi\ratio^2 - \psi_1\psi_2\, ,\\
E_1(\ratio, \psi_1, \psi_2)  \equiv&~ \psi_2\chi^3\ratio^4 - \psi_2\chi^2\ratio^2 + \psi_1\psi_2\chi\ratio^2 - \psi_1\psi_2\, , \\
E_2(\ratio, \psi_1, \psi_2) \equiv&~ \chi^5\ratio^6 - 3\chi^4\ratio^4+ (\psi_1 - 1)\chi^3\ratio^6 + 2\chi^3\ratio^4 + 3\chi^3\ratio^2 + (- \psi_1 - 1)\chi^2\ratio^4 - 2\chi^2\ratio^2 - \chi^2\,.\\
\end{aligned}
\end{equation}
Then the expression for the asymptotic risk of minimum norm interpolator gives 
\[
\cR(\psi_1, \psi_2) = F_1^2\frac{E_1(\ratio, \psi_1, \psi_2) }{E_0 (\ratio, \psi_1, \psi_2)  } + \tau^2\frac{E_2(\ratio, \psi_1, \psi_2) }{E_0 (\ratio, \psi_1, \psi_2)  } + \tau^2.
\]
The expression for the norm of the minimum norm interpolator gives
\[
\begin{aligned}
A_1 =&~ \frac{\rho}{1 + \rho} \Big[ - \chi^2 (\chi \ratio^4 - \chi \ratio^2 + \psi_2 \ratio^2 + \ratio^2 - \chi \psi_2 \ratio^4 + 1)\Big]  + \frac{1}{ 1 + \rho} \Big[ \chi^2 (\chi \ratio^2 - 1) (\chi^2 \ratio^4 - 2 \chi \ratio^2 + \ratio^2 + 1) \Big], \\
A_0 =&~ - \chi^5\ratio^6 + 3\chi^4\ratio^4 + (\psi_1\psi_2 - \psi_2 - \psi_1 + 1)\chi^3\ratio^6 - 2\chi^3\ratio^4 - 3\chi^3\ratio^2\\
&~+ (\psi_1 + \psi_2 - 3\psi_1\psi_2 + 1)\chi^2\ratio^4 + 2\chi^2\ratio^2 + \chi^2 + 3\psi_1\psi_2\chi\ratio^2 - \psi_1\psi_2, \\
\cA(\psi_1, \psi_2) =&~ \psi_1(F_1^2+\tau^2)A_1 /(\mu_\star^2 A_0). 
\end{aligned}
\]
\end{definition}

\section{Experimental setup for simulations in Figure \ref{fig:simulation}}
\label{sec:simulation_detail}
In this section, we present additional details for Figure \ref{fig:simulation}. 
We choose $y_i = \< \bx_i, \bbeta\>$ for some $\Vert \bbeta \Vert_2^2 = 1$, the ReLU activation function $\sigma(x) = \max\{ x, 0\}$, and $\psi_1=N/d = 2.5$ and $\psi_2=n/d = 1.5$. 

For the theoretical curves (in solid lines), we choose $\lambda\in[0.426, 2]$, so that $\cA_U(\lambda) \in [0, 15]$, and plot the parametric curve $(\cA_U(\lambda), \overline\cU(\lambda)+\lambda\cA_U(\lambda))$ for the uniform convergence. For the uniform convergence over interpolators, we choose  $\lambda\in[0.21, 2]$ so that $\cA_T(\lambda) \in [6.4, 15]$, and plot $(\cA_T(\lambda), \overline\cT(\lambda)+\lambda\cA_T(\lambda))$. The definitions of these theoretical predictions are given in Definition \ref{def:analytic_expression_overline}, Remark \ref{rmk:simplification} and Remark \ref{rmk:simplification_interpolating} 

For the empirical simulations (in dots), first recall that in Proposition \ref{prop:concentration_lag}, we defined
\[
\begin{aligned}
\ba_U(\lambda) =&~ \arg\max_{\ba}  \Big[ R(\ba) - \what R_n(\ba) - \psi_1 \lambda \| \ba \|_2^2 \Big],\\
\ba_T(\lambda) =&~ \arg\max{\ba}\inf_{\bmu} \Big[ R(\ba) - \lambda\psi_1 \| \ba \|_2^2 + 2 \< \bmu, \bZ \ba - \by/\sqrt{d}\ \>\Big]. 
\end{aligned}
\]
After picking a value of $\lambda$, we sample $20$ independent problem instances, with the number of features $N=500$, number of samples $n=300$, covariate dimension $d=200$. We compute the corresponding $(\psi_1\|\ba_U\|_2^2, R(\ba_U)-\hat{R}_n(\ba_U))$ and $(\psi_1\|\ba_T\|_2^2, R(\ba_T))$ for each instance. Then, we plot the empirical mean and $1/\sqrt{20}$ times the empirical standard deviation (around the mean) of each coordinate. 

%% file: sections/proofs.tex
\def\obM{{\overline \bM}}
\def\obv{{\overline \bv}}
\def\bE{{\boldsymbol E}}

\section{Proof of Proposition \ref{prop:strong_duality}}
The proof of Proposition \ref{prop:strong_duality} contains two parts: standard uniform convergence $U$ and uniform convergence over interpolators $T$. The proof for the two cases are essentially the same, both based on the fact that strong duality holds for quadratic program with single quadratic constraint (c.f. \citet{boyd_vandenberghe_2004}, Appendix A.1).

\subsection{Standard uniform convergence $U$}
Recall that the uniform convergence bound $U$ is defined as in Eq. \eqref{eqn:uniform}
\[
	U(A, N, n, d) =~ \sup_{(N/d) \| \ba \|_2^2 \le A} \Big( R(\ba) - \what R_n(\ba) \Big).
\]
Since the maximization problem in \eqref{eqn:uniform} is a quadratic program with a single quadratic constraint, the strong duality holds. So we have
\[
\sup_{(N/d) \| \ba \|_2^2 \le A^2} R(\ba) - \what R_n(\ba) = \inf_{\lambda\ge0} \sup_{\ba} \Big[ R(\ba) - \what{R}_n(\ba) - \psi_1\lambda(\|\ba\|_2^2 - \psi_1^{-1}A) \Big].
\]
Finally, by the definition of $\overline U$ as in Eq. (\ref{eqn:uniform_lag}), we get
\[
U(A, N, n, d) = \inf_{\lambda\ge0}\Big[ \overline U(\lambda, N, n, d) + \lambda A\Big]. 
\]

\subsection{Uniform convergence over interpolators $T$}

Without loss of generality, we consider the regime when $N > n$. 

Recall that the uniform convergence over interpolators $T$ is defined as in Eq. \eqref{eqn:uniform_zeroloss}
\[
T(A, N, n, d)= \sup_{(N/d) \| \ba \|_2^2 \le A, \what R_n(\ba) = 0} R(\ba). 
\]
When the set $\{ \ba \in \R^N : (N / d) \| \ba \|_2^2 \le A, \what R_n(\ba) = 0 \}$ is empty, we have 
\[
T(A, N, n, d) =~ \inf_{\lambda \ge 0}  \Big[ \overline T(\lambda, N, n, d) + \lambda A \Big]=-\infty.
\]
In the following, we assume that the set $\{  \ba \in \R^N : (N / d) \| \ba \|_2^2 \le A, \what R_n(\ba) = 0 \}$ is non-empty, i.e., there exists $\ba \in \R^N$ such that $\what R_n(\ba) = 0$ and $(N/ d) \| \ba \|_2^2 \le A$. 

Let $m$ be the dimension of the null space of $\bZ \in \R^{n \times N}$, i.e. $m = \dim(\{ \bu: \bZ \bu = \bzero \})$. Note that $\bZ \in \R^{N \times n}$ and $N > n$, we must have $N - n \le m \le N$. We let $\bR \in \R^{N \times m}$ be a matrix whose column space gives the null space of matrix $\bZ$. Let $\ba_0$ be the minimum norm interpolating solution (whose existence is given by the assumption that $\{ \ba \in \R^N: \what R_n(\ba) = 0\}$ is non-empty)
\[
\ba_0 = \lim_{\lambda \to 0_+} \arg \min_{\ba \in \R^N} \Big[ \what R_n(\ba) + \lambda \| \ba \|_2^2 \Big] = \arg \min_{\ba \in \R^N: \what R_n(\ba) = 0} \| \ba \|_2^2. 
\]
Then we have
\[
\{\ba\in\R^N: \what R_n(\ba) = 0\} = \{\ba\in\R^N: \by = \sqrt{d}\bZ\ba\} = \{\bR\bu + \ba_0: \bu\in\R^m\}.
\]
Then $T$ can be rewritten as a maximization problem in terms of $\bu$: 
\[
\begin{aligned}
\sup_{(N/d) \| \ba \|_2^2 \le A, \what R_n(\ba) = 0} R(\ba) =&~ \sup_{\bu \in \R^{m}: \| \bR \bu + \ba_0 \|_2^2 \le \psi_1^{-1} A} \Big[\<\bR \bu + \ba_0, \bU (\bR \bu + \ba_0)\> - 2\<\bR \bu + \ba_0, \bv\> + \E(y^2) \Big]\\
=&~ R(\ba_0) +\sup_{\bu \in \R^{m}: \| \bR \bu + \ba_0 \|_2^2 \le \psi_1^{-1} A} \Big[\<\bu, \bR^\sT \bU \bR \bu\> + 2\<\bR \bu, \bU \ba_0 - \bv\>\Big].
\end{aligned}
\]
Note that the optimization problem only has non-feasible region when $A>(N/d)\|\ba_0\|_2^2$. By strong duality of quadratic programs with a single quadratic constraint, we have
\[
\begin{aligned}
&~\sup_{\bu \in \R^{m}: \| \bR \bu + \ba_0 \|_2^2 \le \psi_1^{-1} A} \Big[\<\bu, \bR^\sT \bU \bR \bu\> + 2\<\bR \bu, \bU \ba_0 - \bv\>\Big] \\
=&~ \inf_{\lambda\geq 0}\sup_{\bu \in \R^{m}} \Big[\<\bu, \bR^\sT \bU \bR \bu\> + 2\<\bR \bu, \bU \ba_0 - \bv\> - \lambda(\psi_1\|\bR\bu+\ba_0\|_2^2 - A)\Big].
\end{aligned}
\]
The maximization over $\bu$ can be restated as the maximization over $\ba$:
\[
\begin{aligned}
R(\ba_0) + \sup_{\bu \in \R^{m}} \Big[\<\bu, \bR^\sT \bU \bR \bu\> + 2\<\bR \bu, \bU \ba_0 - \bv\> - \lambda\psi_1\|\bR\bu+\ba_0\|_2^2 \Big] = \sup_{\ba: \what R_n(\ba) = 0} \Big[R(\ba) - \lambda\psi_1\|\ba\|_2^2\Big].
\end{aligned}
\]
Moreover, since $\sup_{\ba: \what R_n(\ba) = 0} [ R(\ba) - \lambda\psi_1\|\ba\|_2^2 ]$ is a quadratic programming with linear constraints, we have
\[
\sup_{\ba: \what R_n(\ba) = 0} \Big[ R(\ba) - \lambda\psi_1\|\ba\|_2^2 \Big]  =  \sup_{\ba}\inf_{\bmu} \Big[ R(\ba) - \lambda\psi_1 \| \ba \|_2^2 + 2 \< \bmu, \bZ \ba - \by/\sqrt{d}\ \>\Big]. 
\]
Combining all the equality above and the definition of $\overline T$ as in Eq. (\ref{eqn:uniform_zeroloss_lag}), we have
\[
\begin{aligned}
T(A, N, n, d) =&~ \sup_{(N/d) \| \ba \|_2^2 \le A, \what R_n(\ba)
= 0} R(\ba) \\
=&~ R(\ba_0) +\sup_{\bu \in \R^{m}: \| \bR \bu + \ba_0 \|_2^2 \le \psi_1^{-1} A} \Big[ \<\bu, \bR^\sT \bU \bR \bu\> + 2\<\bR \bu, \bU \ba_0 - \bv\> \Big]\\
=&~R(\ba_0) + \inf_{\lambda\geq0}\sup_{\bu} \Big[\<\bu, \bR^\sT \bU \bR \bu\> + 2\<\bR \bu, \bU \ba_0 - \bv\> - \lambda(\psi_1\|\bR\bu+\ba_0\|_2^2 - A) \Big]\\
=&~\inf_{\lambda\geq0} \Big\{ \lambda A + R(\ba_0) +  \sup_{\bu} \big[\<\bu, \bR^\sT \bU \bR \bu\> + 2\<\bR \bu, \bU \ba_0 - \bv\> - \lambda\psi_1\|\bR\bu+\ba_0\|_2^2\Big] \Big\} \\
=&~  \inf_{\lambda\geq0} \Big\{\lambda A +\sup_{\ba: \what R_n(\ba) = 0} \Big[ R(\ba) - \lambda\psi_1\|\ba\|_2^2 \Big] \Big\}\\
=&~ \inf_{\lambda\geq0} \Big\{ \lambda A +\sup_{\ba}\inf_{\bmu} \Big[ R(\ba) - \lambda\psi_1 \| \ba \|_2^2 + 2 \< \bmu, \bZ \ba - \by/\sqrt{d}\ \>\Big] \Big\}\\
=&~ \inf_{\lambda\geq 0 } \Big[ \overline T(\lambda, N, n, d) +  \lambda A\Big].
\end{aligned}
\]
This concludes the proof. 

\section{Proof of Proposition \ref{prop:concentration_lag}}

Note that the definitions of $\overline U$ and $\overline T$ as in Eq. (\ref{eqn:uniform_lag}) and (\ref{eqn:uniform_zeroloss_lag}) depend on $\bbeta = \bbeta^{(d)}$, where $\bbeta^{(d)}$ gives the coefficients of the target function $f_d(\bx) = \< \bx, \bbeta^{(d)}\>$. Suppose we explicitly write their dependence on $\bbeta = \bbeta^{(d)}$, i.e., $\overline U(\lambda, N, n, d) = \overline U(\bbeta, \lambda, N, n, d)$ and $\overline T(\lambda, N, n, d)= \overline T(\bbeta, \lambda, N, n, d)$, then we can see that for any fixed $\bbeta_\star$ and $\tilde \bbeta$ with $\| \tilde \bbeta \|_2 = \| \bbeta_\star \|_2$, we have $\overline U(\bbeta_\star, \lambda, N, n, d) \stackrel{d}{=} \overline U(\tilde \bbeta, \lambda, N, n, d)$ and $\overline T(\bbeta_\star, \lambda, N, n, d) \stackrel{d}{=} \overline T(\tilde \bbeta, \lambda, N, n, d)$ where the randomness comes from $\bX, \bTheta, \beps$. This is by the fact that the distribution of $\bx_i$'s and $\btheta_a$'s are rotationally invariant. As a consequence, for any fixed deterministic $\bbeta_\star$, if we take $\bbeta \sim \Unif(\S^{d-1}(\| \bbeta_\star\|_2))$, we have
\[
\begin{aligned}
\overline U(\bbeta_\star, \lambda, N, n, d) \stackrel{d}{=}&~ \overline U(\bbeta, \lambda, N, n, d), \\
\overline T(\bbeta_\star, \lambda, N, n, d) \stackrel{d}{=}&~ \overline T(\bbeta, \lambda, N, n, d). \\
\end{aligned}
\]
where the randomness comes from $\bX, \bTheta, \beps, \bbeta$. 

Consequently, as long as we are able to show the equation
\[
 \overline U(\bbeta, \lambda, N, n, d) = \overline\cU(\lambda, \psi_1, \psi_2) + o_{d, \P}(1)
\]
for random $\bbeta \sim \Unif(\S^{n-1}(\normf_1))$, this equation will also hold for any deterministic $ \bbeta_\star$ with $\| \bbeta_\star \|_2^2 = \normf_1^2$. Vice versa for $\overline T$, $\| \overline \ba_U \|_2^2 $ and $\| \overline \ba_T \|_2^2$. 

As a result, in the following, we work with the assumption that $\bbeta = \bbeta^{(d)} \sim \Unif(\S^{d-1}(\normf_1))$. That is, in proving Proposition \ref{prop:concentration_lag}, we replace Assumption \ref{ass:linear_target} by Assumption \ref{ass:linear_target_prime} below. By the argument above, as long as Proposition \ref{prop:concentration_lag} holds under Assumption \ref{ass:linear_target_prime}, it also holds under the original assumption, i.e., Assumption \ref{ass:linear_target}. 
\begin{assumption}[Linear Target Function]\label{ass:linear_target_prime} 
We assume that $f_d\in L^2(\S^{d-1}(\sqrt{d}))$ with $f_d(\bx) = \< \bbeta^{(d)}, \bx\>$, where $\bbeta^{(d)} \sim \Unif(\S^{d-1}(\normf_1))$.
\end{assumption}

\subsection{Expansions}

Denote $\bv = (v_i)_{i \in [N]} \in \R^N$ and $\bU = (U_{ij})_{i, j \in [N]} \in \R^{N \times N}$ where their elements are defined via
\[
\begin{aligned}
v_i \equiv&~ \E_{\eps, \bx}[y \sigma(\< \bx, \btheta_i\> / \sqrt d)], \\
U_{ij} \equiv&~ \E_\bx[\sigma(\< \bx, \btheta_i\> / \sqrt d) \sigma(\< \bx, \btheta_j\> / \sqrt d)].  
\end{aligned}
\]
Here, $y = \< \bx, \bbeta\> + \eps$, where $\bbeta \sim \Unif(\S^{d-1}(\normf_1))$, $\bx \sim \Unif(\S^{d-1}(\sqrt{d}))$, $\eps \sim \cN(0, \tau^2)$, and  $(\btheta_j)_{j \in [N]} \sim_{iid} \Unif(\S^{d-1}(\sqrt{d}))$ are mutually independent. The expectations are taken with respect to the test sample $\bx \sim \Unif(\S^{d-1}(\sqrt{d}))$ and $\eps \sim \cN(0, \tau^2)$ (especially, the expectations are conditional on $\bbeta$ and $(\btheta_i)_{i \in [N]}$).

Moreover, we denote $\by = (y_1, \ldots, y_n)^\sT \in \R^n$ where $y_i = \< \bx_i, \bbeta\> + \beps_i$. Recall that $(\bx_i)_{i \in [n]} \sim_{iid} \Unif(\S^{d-1}(\sqrt{d}))$ and $(\eps_i)_{i \in [n]} \sim_{iid} \cN(0, \tau^2)$ are mutually independent and independent from $\bbeta \sim \Unif(\S^{d-1}(\sqrt{d}))$. We further denote $\bZ = (Z_{ij})_{i \in [n], j \in [N]}$ where its elements are defined via 
\[
Z_{ij} = \sigma(\< \bx_i, \btheta_j\> / \sqrt d) / \sqrt d. 
\]

The population risk \eqref{eqn:pop_risk} can be reformulated as 
\[
R(\ba) = \<\ba, \bU \ba\> - 2\<\ba, \bv\> + \E[y^2],
\]
where $\ba = (a_1, \dots, a_N)\in\R^N$. The empirical risk \eqref{eqn:emp_risk} can be reformulated as 
\[
\what R_n(\ba) =  \psi_2^{-1} \<\ba, \bZ^\sT\bZ\ba\> - 2\psi_2^{-1}\frac{\<\bZ^\sT\by, \ba\>}{\sqrt{d}} + \frac{1}{n}\|\by\|_2^2.
\]

By the Appendix A in \citet{mm19} (we include in the Appendix \ref{sec:Background} for completeness), we can expand $\sigma(x)$ in terms of Gegenbauer polynommials
\[
\begin{aligned}
\sigma(x) =&~ \sum_{k = 0}^\infty \lambda_{d, k}(\sigma) B(d, k) Q_k^{(d)}(\sqrt d \cdot x),\\
\end{aligned}
\]
where $Q_k^{(d)}$ is the $k$'th Gegenbauer polynomial in $d$ dimensions, $B(d, k)$ is the dimension of the space of polynomials on $\S^{d-1}(\sqrt{d})$ with degree exactly $k$. Finally, $\lambda_{d, k}(\sigma)$ is the $k$'th Gegenbauer coefficient. More details of this expansion can be found in Appendix \ref{sec:Background}. 

By the properties of Gegenbauer polynomials (c.f. Appendix \ref{sec:Gegenbauer}), we have
\[
\begin{aligned}
\E_{ \bx \sim \Unif(\S^{d-1}(\sqrt{d}))}[\bx Q_k(\< \bx, \btheta_i\>)] =&~ \bzero, ~~~~~ &\forall k \neq&~ 1,\\
\E_{ \bx \sim \Unif(\S^{d-1}(\sqrt{d}))}[\bx Q_1(\< \bx, \btheta_i\>)] =&~ \btheta_i /d, ~~~~~ &k =&~ 1.\\
\end{aligned}
\]
As a result, we have
\begin{align}
v_i =&~ \E_{\eps, \bx}[y \sigma(\< \bx, \btheta_i\> / \sqrt d)] = \sum_{k = 0}^\infty \lambda_{d, k}(\sigma) B(d, k) \E_{\bx}[\< \bx, \bbeta\> Q_k^{(d)}(\sqrt d \cdot x) ] = \lambda_{d, 1}(\sigma) \<\btheta_i, \bbeta \>. \label{eqn:vexpan}
\end{align}

\subsection{Removing the perturbations}\label{sec:removing_perturbation}

By Lemma \ref{lem:decomposition_of_kernel_matrix} and \ref{lem:small_lambda_d0} as in Appendix \ref{sec:auxiliary_lemmas}, we have the following decomposition
\begin{equation}\label{eqn:Uexpan}
\bU =  \ob_1^2 \bQ  + \ob_\star^2 \id_N + \bDelta,
\end{equation}
with $\bQ = \bTheta \bTheta^\sT / d$, $\E[\| \bDelta \|_{\op}^2] = o_{d}(1)$, and $\ob_1^2$ and $\ob_\star^2$ are given in Assumption \ref{ass:activation}. 

In the following, we would like to show that $\bDelta$ has vanishing effects in the asymptotics of $\overline U$, $\overline T$, $\| \overline \ba_U \|_2^2$ and $\| \overline \ba_T \|_2^2$.

For this purpose, we denote
\begin{equation}\label{eqn:definitions_oUc_oTc}
\begin{aligned}
\bU_c =&~  \ob_1^2 \bQ  + \ob_\star^2 \id_N, \\
R_c(\ba) =&~ \<\ba, \bU_c \ba\> - 2\<\ba, \bv\> + \E[y^2], \\
\what R_{c, n}(\ba) =&~ \<\ba, \psi_2^{-1} \bZ^\sT \bZ \ba\> - 2\<\ba, \psi_2^{-1} \bZ^\sT \by / \sqrt{d}\> + \E[y^2], \\
\overline U_c(\lambda, N, n, d) =&~  \sup_{\ba} \Big( R_c(\ba) - \what R_{c, n}(\ba) - \psi_1 \lambda \| \ba \|_2^2 \Big), \\
\overline T_c(\lambda, N, n, d) =&~ \sup_{\ba}\inf_{\bmu} \Big[ R_c(\ba) - \lambda\psi_1 \| \ba \|_2^2  + 2 \< \bmu, \bZ \ba - \by/\sqrt{d}\ \>\Big]. 
\end{aligned}
\end{equation}
For a fixed $\lambda \in \Lambdau$, note we have 
\begin{equation}\label{eqn:oUc_in_proof}
\begin{aligned}
\overline U_c(\lambda, N, n, d) =&~  \sup_{\ba}\Big(\< \ba, (\bU_c - \psi_2^{-1}\bZ^\sT \bZ - \psi_1 \lambda \Id_N) \ba\> - 2 \< \ba, \bv - \psi_2^{-1} \frac{\bZ^\sT \by}{\sqrt{d}}\> \Big) \\
=&~ \sup_{\ba}\Big(\< \ba, \obM \ba\> - 2 \< \ba, \obv\> \Big)
\end{aligned}
\end{equation}
where $\obM = \bU_c - \psi_2^{-1}\bZ^\sT \bZ - \psi_1 \lambda \Id_N$ and $\obv = \bv - \psi_2^{-1} \bZ^\sT \by /\sqrt{d}$. When $\bX, \bTheta$ are such that the good event in Assumption \ref{ass:overline_U_invertable} happens (which says that $\obM \preceq - \eps \id_N$ for some $\eps > 0$), the inner maximization can be uniquely achieved at 
\begin{equation}\label{eqn:oaUc_in_proof}
 \overline\ba_{U, c}(\lambda) = \argmax_{\ba}\Big(\< \ba, \obM \ba\> - 2 \< \ba, \obv\> \Big) = {\obM}^{-1} \obv. 
\end{equation}
and when the good event $\{ \| \bDelta \|_{\op} \le \eps / 2\}$ also happens, the maximizer in the definition of $\overline U(\lambda, N, n, d)$ (c.f. Eq. (\ref{eqn:uniform_lag})) can be uniquely achieved at
\[
\overline\ba_U(\lambda) = \argmax_{\ba}\Big(\< \ba, (\obM + \bDelta) \ba\> - 2 \< \ba, \obv\> \Big) = (\obM + \bDelta)^{-1} \obv. 
\]
Note we have
\[
 \overline\ba_U(\lambda) -  \overline\ba_{U, c}(\lambda) = (\obM + \bDelta)^{-1} \obv - \obM^{-1} \obv = (\obM + \bDelta)^{-1} \bDelta \obM^{-1} \obv, 
\]
so by the fact that $\| \bDelta \|_{\op} = o_{d, \P}(1)$, we have
\[
\| \overline\ba_U(\lambda) -  \overline\ba_{U, c}(\lambda)\|_2 \le  \| (\obM + \bDelta)^{-1} \bDelta \|_{\op} \| \overline \ba_{U, c}(\lambda) \|_2 = o_{d, \P}(1) \| \overline \ba_{U, c}(\lambda) \|_2. 
\]
This gives $\| \overline\ba_U(\lambda) \|_2^2 = (1 + o_{d, \P}(1)) \| \overline \ba_{U, c}(\lambda) \|_2^2$. 

Moreover, by the fact that $\| \bDelta \|_{\op} = o_{d, \P}(1)$, we have 
\[
\begin{aligned}
\overline U_c(\lambda, N, n, d) =&~  \sup_{\ba} \Big( R(\ba) - \what R_n(\ba) - \psi_1 \lambda \| \ba \|_2^2 - \< \ba, \bDelta \ba\> \Big) + \E[y^2] - \| \by \|_2^2/n\\
=&~ \overline U(\lambda, N, n, d) + o_{d, \P}(1) (\|  \overline\ba_{U, c}(\lambda) \|_2^2 + 1). \\
\end{aligned}
\]
As a consequence, as long as we can prove the asymptotics of $\overline U_c$ and $\|  \overline\ba_{U, c}(\lambda) \|_2^2$, it also gives the asymptotics of $\overline U$ and $\| \overline\ba_U(\lambda) \|_2^2$. Vice versa for $\overline T$ and $\| \overline \ba_T(\lambda) \|_2^2$. 


\subsection{The asymptotics of $\overline U_c$ and $\psi_1\| \overline \ba_{U, c}(\lambda) \|_2^2$}

In the following, we derive the asymptotics of $\overline U_c(\lambda, N, n, d)$ and $\psi_1\| \overline \ba_{U, c}(\lambda) \|_2^2$. When we refer to $\overline \ba_{U, c}(\lambda)$, it is always well defined with high probability, since it can be well defined under the condition that the good event in Assumption \ref{ass:overline_U_invertable} happens. Note that this good event only depend on $\bX, \bTheta$ and is independent of $\bbeta, \beps$. 

By Eq. (\ref{eqn:oUc_in_proof}) and (\ref{eqn:oaUc_in_proof}), simple calculation shows that
\[
\begin{aligned}
\overline U_c(\lambda, N, n, d) \equiv&~ - \< \obv, \obM^{-1} \obv\>  = - \Psi_1 - \Psi_2 -  \Psi_3,\\
\|  \overline\ba_{U, c} \|_2^2 \equiv&~ \< \obv, \obM^{-2} \obv\>  = \Phi_1 + \Phi_2 + \Phi_3, \\
\end{aligned}
\]
where
\[
\begin{aligned}
\Psi_1 =&~ \<\bv,  \obM^{-1} \bv\>,  & \Phi_1 =&~ \<\bv,  \obM^{-2} \bv\>, \\
\Psi_2 =&~  -2\psi_2^{-1} \<\frac{\bZ^\sT \by}{\sqrt{d}},  \obM^{-1}  \bv\>,  ~~~~~~~~~& \Phi_2=&~ -2\psi_2^{-1} \<\frac{\bZ^\sT \by}{\sqrt{d}},  \obM^{-2}  \bv\>,\\
\Psi_3 =&~ \psi_2^{-2} \<\frac{\bZ^\sT \by}{\sqrt{d}},  \obM^{-1}  \frac{\bZ^\sT \by}{\sqrt{d}} \>, & \Phi_3=&~\psi_2^{-2} \<\frac{\bZ^\sT \by}{\sqrt{d}},  \obM^{-2}  \frac{\bZ^\sT \by}{\sqrt{d}} \>. 
\end{aligned}
\]
The following lemma gives the expectation of $\Psi_i$'s and $\Phi_i$'s with respect to $\bbeta$ and $\beps$. 
\begin{lemma}[Expectation of $\Psi_i$'s and $\Phi_i$'s]\label{lem:expectation_Psi_Phi_U}
Denote $\bq_U(\lambda, \bpsi) = (\mu_\star^2 - \lambda\psi_1, \mu_1^2, \psi_2, 0,0)$. We have
\[
\begin{aligned}
\E_{\beps, \bbeta}[\Psi_1] =&~ \mu_1^2 \normf_1^2 \cdot \frac{1}{d}  \Trace\Big(\obM^{-1} \bQ \Big) \times (1 + o_d(1)), \\
\E_{\beps, \bbeta}[\Psi_2] =&~ -\frac{2 \normf_1^2}{\psi_2} \cdot \frac{1}{d}\Trace\Big( \bZ \obM^{-1} \bZ_1^\sT \Big) \times (1 + o_d(1)), \\
\E_{\beps, \bbeta}[\Psi_3] =&~ \frac{\normf_1^2}{\psi_2^2} \cdot \frac{1}{d} \Trace\Big( \bZ\obM^{-1}\bZ^\sT \bH \Big) + \frac{\tau^2}{\psi_2^2}  \cdot \frac{1}{d}\Trace\Big( \bZ\obM^{-1}\bZ^\sT \Big), \\
\E_{\beps, \bbeta}[\Phi_1] =&~ \mu_1^2 \normf_1^2 \cdot \frac{1}{d} \Trace\Big(\obM^{-2} \bQ \Big) \times (1 + o_d(1)), \\
\E_{\beps, \bbeta}[\Phi_2] =&~ -\frac{2 \normf_1^2}{\psi_2} \cdot \frac{1}{d}\Trace\Big( \bZ \obM^{-2} \bZ_1^\sT \Big) \times (1 + o_d(1)), \\
\E_{\beps, \bbeta}[\Phi_3] =&~ \frac{\normf_1^2}{\psi_2^2} \cdot \frac{1}{d} \Trace\Big( \bZ\obM^{-2}\bZ^\sT \bH \Big) + \frac{\tau^2}{\psi_2^2}  \cdot \frac{1}{d}\Trace\Big( \bZ\obM^{-2}\bZ^\sT \Big).
\end{aligned}
\]
Here the definitions of $\bQ$, $\bH$, and $\bZ_1$ are given by Eq. (\ref{eqn:def_Q_H_Z_Z1}). 

Furthermore, we have
\[
\begin{aligned}
\E_{\beps, \bbeta}[\Psi_1] =&~ \mu_1^2 \normf_1^2 \cdot \partial_{s_2} G_d(0_+; \bq_U(\lambda, \bpsi)) \times (1 + o_d(1)), \\
\E_{\beps, \bbeta}[\Psi_2] =&~  \normf_1^2 \cdot \partial_{p} G_d(0_+; \bq_U(\lambda, \bpsi)) \times (1 + o_d(1)), \\
\E_{\beps, \bbeta}[\Psi_3] 
=&~ \normf_1^2\cdot (\partial_{t_2} G_d(0_+; \bq_U(\lambda, \bpsi)) - 1)  + \tau^2  \cdot (\partial_{t_1} G_d(0_+; \bq_U(\lambda, \bpsi)) - 1), \\
\E_{\beps, \bbeta}[\Phi_1] =&~ - \mu_1^2 \normf_1^2 \cdot \partial_{s_1 }\partial_{s_2} G_d(0_+; \bq_U(\lambda, \bpsi)) \times (1 + o_d(1)), \\
\E_{\beps, \bbeta}[\Phi_2] =&~ -  \normf_1^2 \cdot \partial_{s_1} \partial_{p} G_d(0_+; \bq_U(\lambda, \bpsi)) \times (1 + o_d(1)), \\
\E_{\beps, \bbeta}[\Phi_3] =&~  - \normf_1^2\cdot \partial_{s_1} \partial_{t_2} G_d(0_+; \bq_U(\lambda, \bpsi)) - \tau^2  \cdot \partial_{s_1} \partial_{t_1} G_d(0_+; \bq_U(\lambda, \bpsi)). 
\end{aligned}
\]
The definition of $G_d$ is as in Definition \ref{def:log_determinant_A}, and $\nabla_\bq^k G_d(0_+; \bq)$ for $k \in \{1, 2\}$ stands for the $k$'th derivatives (as a vector or a matrix) of $G_d(\bi u; \bq)$ with respect to $\bq$ in the $u \to 0+$ limit (with its elements given by partial derivatives)
\[
\nabla_\bq^k G_d(0_+; \bq) = \lim_{u \to 0+} \nabla_\bq^k G_d(\bi u; \bq). 
\]
\end{lemma}

We next state the asymptotic characterization of the log-determinant which was proven in \cite{mm19}. 

\begin{proposition}[Proposition 8.4 in \cite{mm19}]\label{prop:expression_for_log_determinant}
Define
\begin{equation}
\begin{aligned}
\Xi(\xi, z_1, z_2; \bq; \bpsi) \equiv&~ \log[(s_2 z_1 + 1)(t_2 z_2 + 1) - \ob_1^2 (1 + p)^2 z_1 z_2] - \ob_\star^2 z_1 z_2 \\
 &+ s_1 z_1 +  t_1 z_2  - \psi_1 \log (z_1 / \psi_1) - \psi_2 \log (z_2 / \psi_2)  - \xi (z_1 + z_2) - \psi_1 - \psi_2.
\end{aligned}
\end{equation}
For $\xi \in \C_+$ and $\bq \in \cQ$ (c.f. Eq. (\ref{eqn:definition_of_cQ})), let $m_1(\xi; \bq; \bpsi), m_2(\xi; \bq; \bpsi)$ be defined as the analytic continuation of solution of Eq. (\ref{eq:FixedPoint}) as defined in Definition \ref{def:Stieltjes}. Define
\begin{equation}
g(\xi; \bq; \bpsi) = \Xi(\xi, m_1(\xi; \bq; \bpsi), m_2(\xi; \bq; \bpsi); \bq; \bpsi). 
\end{equation}
Consider proportional asymptotics $N/d\to\psi_1$,  $N/d\to\psi_2$,  
as per Assumption \ref{ass:linear}. Then for any fixed $\xi \in \C_+$ and $\bq \in \cQ$, we have
\begin{equation}\label{eqn:expression_for_log_determinant}
\begin{aligned}
\lim_{d \to \infty} \E[ \vert G_d(\xi; \bq) -  g(\xi; \bq; \bpsi) \vert] = 0.
\end{aligned}
\end{equation}
Moreover, for any fixed $u \in \R_+$ and $\bq \in \cQ$, we have
\begin{align}
\lim_{d \to \infty}  \E[\| \partial_\bq G_d(\bi u; \bq) - \partial_\bq g(\bi u; \bq; \bpsi) \|_2 ] =&~ 0, \label{eqn:convergence_of_derivatives}\\
\lim_{d \to \infty}  \E[\| \nabla_{\bq}^2 G_d(\bi u; \bq) - \nabla_{\bq}^2 g(\bi u; \bq; \bpsi) \|_{\op} ] =&~ 0. \label{eqn:convergence_of_second_derivatives}
\end{align}
\end{proposition}

\begin{remark}
Note that Proposition 8.4 in \cite{mm19} stated that the Eq. (\ref{eqn:convergence_of_derivatives}) and (\ref{eqn:convergence_of_second_derivatives}) holds at $\bq = \bzero$. However, by a simple modification of their proof, one can show that these equations also holds at any $\bq \in \cQ$. 
\end{remark}

Combining Assumption \ref{ass:exchange_limit} with Proposition \ref{prop:expression_for_log_determinant}, we have 
\begin{proposition}\label{prop:exchange_limit_g_U}
Let Assumption \ref{ass:exchange_limit} holds. For any $\lambda \in \Lambdau$, denote $\bq_U = \bq_U(\lambda, \bpsi) = (\mu_\star^2 - \lambda\psi_1, \mu_1^2, \psi_2, 0,0)$, then we have, for $k = 1,2$, 
\[
\| \nabla_\bq^k  G_d(0_+;\bq_U) - \lim_{u \to 0_+} \nabla_\bq^k g(\bi u; \bq_U; \bpsi) \| = o_{d, \P}(1).
\]
\end{proposition}

As a consequence of Proposition \ref{prop:exchange_limit_g_U}, we can calculate the asymptotics of $\Psi_i$'s and $\Phi_i$'s. Combined with the concentration result in Lemma \ref{lem:concentration_beta_eps_U} latter in the section, the proposition below completes the proof of the part of Proposition \ref{prop:concentration_lag} regarding the standard uniform convergence $U$. Its correctness follows directly from Lemma \ref{lem:expectation_Psi_Phi_U} and Proposition \ref{prop:exchange_limit_g_U}.

\begin{proposition}\label{prop:asymptotics_barU_A^2}
Follow the assumptions of Proposition \ref{prop:concentration_lag}. For any $\lambda \in \Lambdau$, denote $\bq_U(\lambda, \bpsi) = (\mu_\star^2 - \lambda\psi_1, \mu_1^2, \psi_2, 0,0)$, then we have 
\[
\begin{aligned}
\E_{\beps, \bbeta}[\Psi_1] \stackrel{\P}{\to}&~ \mu_1^2 \normf_1^2 \cdot  \partial_{s_2} g(0_+; \bq_U(\lambda, \bpsi); \bpsi), \\
\E_{\beps, \bbeta}[\Psi_2] \stackrel{\P}{\to}&~ \normf_1^2 \cdot \partial_{p} g(0_+; \bq_U(\lambda, \bpsi); \bpsi), \\
\E_{\beps, \bbeta}[\Psi_3] \stackrel{\P}{\to}&~ \normf_1^2  \cdot \Big(\partial_{t_2} g(0_+; \bq_U(\lambda, \bpsi); \bpsi) - 1 \Big) + \tau^2 \Big( \partial_{t_1} g(0_+; \bq_U(\lambda, \bpsi); \bpsi) - 1 \Big), \\
\E_{\beps, \bbeta}[\Phi_1] \stackrel{\P}{\to}&~ - \mu_1^2 \normf_1^2 \cdot \partial_{s1}\partial_{s2} g(0_+; \bq_U(\lambda, \bpsi); \bpsi) , \\
\E_{\beps, \bbeta}[\Phi_2] \stackrel{\P}{\to}&~ -  \normf_1^2 \cdot \partial_{s1}\partial_p g(0_+; \bq_U(\lambda, \bpsi); \bpsi), \\
\E_{\beps, \bbeta}[\Phi_3] \stackrel{\P}{\to}&~ - \normf_1^2 \cdot \partial_{s_1}\partial_{t_2} g(0_+; \bq_U(\lambda, \bpsi); \bpsi)  - \tau^2 \cdot \partial_{s_1}\partial_{t_1} g(0_+; \bq_U(\lambda, \bpsi); \bpsi), \\
\end{aligned}
\]
where $\nabla_\bq^k g(0_+; \bq; \bpsi)$ for $k \in \{1, 2\}$ stands for the $k$'th derivatives (as a vector or a matrix) of $g(\bi u; \bq; \bpsi)$ with respect to $\bq$ in the $u \to 0+$ limit (with its elements given by partial derivatives)
\[
\nabla_{\bq}^k g(0_+; \bq; \bpsi) = \lim_{u \to 0_+} \nabla_{\bq}^k g(\bi u; \bq; \bpsi). 
\]
As a consequence, we have 
\[
 \E_{\beps, \bbeta}[\overline U_c(\lambda, N, n, d)] \stackrel{\P}{\to} \overline \cU(\lambda, \psi_1, \psi_2), ~~~~~ \E_{\beps, \bbeta} [\psi_1\| \overline\ba_{U, c}(\lambda)\|_2^2] \stackrel{\P}{\to} \cA_U(\lambda, \psi_1, \psi_2),
\]
where the definitions of $\overline \cU$ and $\cA_U$ are given in Definition \ref{def:analytic_expression_overline}. 
Here $\stackrel{\P}{\to}$ stands for convergence in probability as $N/d \to \psi_1$ and $n/d \to \psi_2$ (with respect to the randomness of $\bX$ and $\bTheta$). 
\end{proposition}

\begin{lemma}\label{lem:concentration_beta_eps_U}
Follow the assumptions of Proposition \ref{prop:concentration_lag}. For any $\lambda \in \Lambdau$, we have 
\[
\begin{aligned}
\Var_{\beps, \bbeta}[\Psi_1], \Var_{\beps, \bbeta}[\Psi_2], \Var_{\beps, \bbeta}[\Psi_3]  =&~ o_{d, \P}(1), \\
\Var_{\beps, \bbeta}[\Phi_1], \Var_{\beps, \bbeta}[\Phi_2],\Var_{\beps, \bbeta}[\Phi_3]  =&~ o_{d, \P}(1),  
\end{aligned}
\]
so that 
\[
\Var_{\beps, \bbeta}[\overline U_c(\lambda, N, n, d)], \Var_{\beps, \bbeta}[\|  \overline\ba_{U, c}(\lambda) \|_2^2] = o_{d, \P}(1). 
\]
Here, $o_{d, \P}(1)$ stands for converges to $0$ in probability (with respect to the randomness of $\bX$ and $\bTheta$) as $N/d \to \psi_1$ and $n/d\to\psi_2$ and $d \to \infty$. 
\end{lemma}

Now, combining Lemma \ref{lem:concentration_beta_eps_U} and Proposition \ref{prop:asymptotics_barU_A^2}, we have 
\[
\overline U_c(\lambda, N, n, d) \stackrel{\P}{\to} \overline \cU(\lambda, \psi_1, \psi_2), ~~~~~ \psi_1\| \overline\ba_{U, c}(\lambda)\|_2^2 \stackrel{\P}{\to} \cA_U(\lambda, \psi_1, \psi_2),
\]
Finally, combining with the arguments in Appendix \ref{sec:removing_perturbation} proves the asymptotics of $\overline U$ and $\psi_1 \| \overline \ba_U(\lambda) \|_2^2$.

\subsection{The asymptotics of $\overline T_c$ and $\psi_1\| \overline \ba_{T, c}(\lambda) \|_2^2$}

In the following, we derive the asymptotics of $\overline T_c(\lambda, N, n, d)$ and $\psi_1\| \overline \ba_{T, c}(\lambda) \|_2^2$. This follows the same steps as the proof of the asymptotics of $\overline U_c$ and $\psi_1\| \overline \ba_{U, c}(\lambda) \|_2^2$. We will give an overview of its proof. The detailed proof is the same as that of $\overline U_c$, and we will not include them for brevity. 

For a fixed $\lambda \in \Lambdat$, recalling that the definition of $\overline T_c$ as in Eq. (\ref{eqn:definitions_oUc_oTc}), we have 
\begin{equation}\label{eqn:oTc_variational}
\begin{aligned}
\overline T_c(\lambda, N, n, d) =&~ \sup_{\ba}\inf_{\bmu} \Big[ R_c(\ba) - \lambda\psi_1 \| \ba \|_2^2 + 2 \< \bmu, \bZ \ba - \by/\sqrt{d}\ \>\Big] \\
=&~  \sup_{\ba}\inf_{\bmu} \Big( \<\ba, (\bU_c -\lambda\psi_1\bI_N)\ba\> - 2\<\ba,\bv\> + 2\<\bmu, \bZ\ba\>-2\<\bmu,\by/\sqrt{d}\>	\Big) +\E[y^2]\\
=&~ \sup_{\sqrt{d}\bZ\ba=y} \<\ba, (\bU_c-\lambda\psi_1\bI_N)\ba\> - 2\<\ba, \bv\> + \E[y^2]
\end{aligned}
\end{equation}
Whenever the good event in Assumption \ref{ass:overline_U_invertable} happens, $(\bU_c-\lambda\psi_1\bI_N)$ is negative definite in null$(\bZ)$. The optimum of the above variational equation exists. By KKT condition, the optimal $\ba$ and dual variable $\bmu$ satisfies
\begin{itemize}
	\item Stationary condition: $(\bU_c-\lambda\psi_1\bI_N)\ba + \bZ^\sT\mu = \bv$.
	\item Primal Feasible:  $\bZ\ba = \by/\sqrt{d}$.
\end{itemize}
The two conditions can be written compactly as 
\begin{equation}\label{eqn:Tstationary}
\begin{bmatrix}
\bU_c - \psi_1\lambda \id_N & \bZ^\sT \\
\bZ& \bzero 
\end{bmatrix}
\begin{bmatrix}
\ba\\
\bmu
\end{bmatrix}=
\begin{bmatrix}
\bv\\
\by/\sqrt{d}
\end{bmatrix}.
\end{equation}
We define 
\begin{equation}
\obM \equiv \begin{bmatrix}
\bU_c - \psi_1\lambda \id_N & \bZ^\sT \\
\bZ& \bzero 
\end{bmatrix},
~~~~~~~~~~~~~~
\obv \equiv \begin{bmatrix}
\bv\\
\by/\sqrt{d}
\end{bmatrix}.
\end{equation}
Under Assumption \ref{ass:overline_U_invertable}, $\obM$ is invertible. To see this, suppose there exists vector $[\ba_1^\sT, \bmu_1^\sT]^\sT \neq \bzero\in\R^{N+n}$ such that $\obM[\ba_1^\sT, \bmu_1^\sT]^\sT=\bzero$, then
\[
\begin{aligned}
 (\bU_c-\lambda\psi_1\bI_N)\ba_1 + \bZ^\sT\bmu_1 = 0,\\
 \bZ\ba_1 = 0.
\end{aligned}
\]
As in Assumption \ref{ass:overline_U_invertable}, let $\proj_{\nullZ} = \id_N - \bZ^\dagger \bZ$. We write $\ba_1 = \proj_\nullZ \bv_1$ for some $\bv_1\neq\bzero\in\R^N$. Then,
\[
\begin{aligned}
&(\bU_c-\lambda\psi_1\bI_N)\proj_\nullZ \bv_1 + \bZ^\sT\bmu_1=0,\\
\Rightarrow&~ \proj_\nullZ(\bU_c-\lambda\psi_1\bI_N)\proj_\nullZ \bv_1 + \proj_\nullZ\bZ^\sT\bmu_1=0,\\
\Rightarrow&~ \proj_\nullZ(\bU_c-\lambda\psi_1\bI_N)\proj_\nullZ \bv_1 = 0,
\end{aligned}
\]
where the last relation come from the fact that $\bZ\proj_\nullZ=\bzero$. However by Assumption \ref{ass:overline_U_invertable}, $\proj_\nullZ(\bU_c-\lambda\psi_1\bI_N)\proj_\nullZ$ is negative definite, which leads to a contradiction.

In the following, we assume the event in Assumption \ref{ass:overline_U_invertable} happens so that $\obM$ is invertible. In this case, the maximizer in Eq. (\ref{eqn:oTc_variational}) can be well defined as 
\[
\overline\ba_{T, c}(\lambda) = [\id_N, \bzero_{N \times n}]\obM^{-1} \obv. 
\]
Moreover, we can write $\overline T_c$ as
\[
\overline T_c(\lambda, N, n, d) = \E[y^2] - \obv^\sT \obM^{-1} \obv. 
\]
We further define 
\[
\obv_1 = [\bv^\sT, \bzero_{n \times 1}^\sT]^\sT, ~~~~ \obv_2 = [\bzero_{N \times 1}^\sT, \by^\sT / \sqrt{d}]^\sT,~~~~ \bE \equiv \begin{bmatrix} 
\id_N & \bzero_{N \times n} \\
\bzero_{n \times N} & \bzero_{n \times n}
\end{bmatrix}.
\]
Simple calculation shows that
\[
\begin{aligned}
\overline T_c(\lambda, N, n, d) \equiv&~ \E[y^2] - \< \obv, \obM^{-1} \obv\>  = \normf_1^2 + \tau^2 - \Psi_1 - \Psi_2 -  \Psi_3,\\
\|  \overline\ba_{U, c} \|_2^2 \equiv&~ \< \obv, \obM^{-1} \bE \obM^{-1} \obv\>  = \Phi_1 + \Phi_2 + \Phi_3, \\
\end{aligned}
\]
where
\[
\begin{aligned}
\Psi_1 =&~ \<\obv_1,  \obM^{-1} \obv_1\>,  & \Phi_1 =&~ \<\obv_1,  \obM^{-1} \bE \obM^{-1} \obv_1\>, \\
\Psi_2 =&~  2 \<\obv_2,  \obM^{-1}  \obv_1\>,  ~~~~~~~~~& \Phi_2=&~ 2  \<\obv_2,  \obM^{-1} \bE \obM^{-1}  \obv_1\>,\\
\Psi_3 =&~ \<\obv_2,  \obM^{-1} \obv_2\>, & \Phi_3=&~ \< \obv_2,  \obM^{-1} \bE \obM^{-1}  \obv_2 \>. 
\end{aligned}
\]
The following lemma gives the expectation of $\Psi_i$'s and $\Phi_i$'s with respect to $\bbeta$ and $\beps$. 
\begin{lemma}[Expectation of $\Psi_i$'s and $\Phi_i$'s]\label{lem:expectation_Psi_Phi_T}
Denote $\bq_T(\lambda, \bpsi) = (\mu_\star^2 - \lambda\psi_1, \mu_1^2, 0, 0,0)$. We have
\[
\begin{aligned}
\E_{\beps, \bbeta}[\Psi_1] =&~ \mu_1^2 \normf_1^2 \cdot \partial_{s_2} G_d(0_+; \bq_T(\lambda, \bpsi)) \times (1 + o_d(1)), \\
\E_{\beps, \bbeta}[\Psi_2] =&~ \normf_1^2 \cdot \partial_{p} G_d(0_+; \bq_T(\lambda, \bpsi)) \times (1 + o_d(1)), \\
\E_{\beps, \bbeta}[\Psi_3] =&~ \normf_1^2\cdot \partial_{t_2} G_d(0_+; \bq_T(\lambda, \bpsi))   + \tau^2  \cdot \partial_{t_1} G_d(0_+; \bq_T(\lambda, \bpsi)), \\
\E_{\beps, \bbeta}[\Phi_1] =&~ - \mu_1^2 \normf_1^2 \cdot \partial_{s_1 }\partial_{s_2} G_d(0_+; \bq_T(\lambda, \bpsi)) \times (1 + o_d(1)), \\
\E_{\beps, \bbeta}[\Phi_2] =&~ -  \normf_1^2 \cdot \partial_{s_1} \partial_{p} G_d(0_+; \bq_T(\lambda, \bpsi)) \times (1 + o_d(1)), \\
\E_{\beps, \bbeta}[\Phi_3] =&~  - \normf_1^2\cdot \partial_{s_1} \partial_{t_2} G_d(0_+; \bq_T(\lambda, \bpsi)) - \tau^2  \cdot \partial_{s_1} \partial_{t_1} G_d(0_+; \bq_T(\lambda, \bpsi)). 
\end{aligned}
\]
The definition of $G_d$ is as in Definition \ref{def:log_determinant_A}, and $\nabla_\bq^k G_d(0_+; \bq)$ for $k \in \{1, 2\}$ stands for the $k$'th derivatives (as a vector or a matrix) of $G_d(\bi u; \bq)$ with respect to $\bq$ in the $u \to 0+$ limit (with its elements given by partial derivatives)
\[
\nabla_\bq^k G_d(0_+; \bq) = \lim_{u \to 0+} \nabla_\bq^k G_d(\bi u; \bq). 
\]
\end{lemma}
The proof of Lemma \ref{lem:expectation_Psi_Phi_T} follows from direct calculation and is identical to the proof of Lemma \ref{lem:expectation_Psi_Phi_U}. Combining Assumption \ref{ass:exchange_limit} with Proposition \ref{prop:expression_for_log_determinant}, we have 
\begin{proposition}\label{prop:exchange_limit_g_T}
Let Assumption \ref{ass:exchange_limit} holds. For any $\lambda \in \Lambdat$, denote $\bq_T = \bq_T(\lambda, \bpsi) = (\mu_\star^2 - \lambda\psi_1, \mu_1^2, 0, 0,0)$, then we have, for $k = 1,2$, 
\[
\| \nabla_\bq^k  G_d(0_+;\bq_T) - \lim_{u \to 0+} \nabla_\bq^k g(\bi u; \bq_T; \bpsi) \| = o_{d, \P}(1).
\]
\end{proposition}

As a consequence of Proposition \ref{prop:exchange_limit_g_T}, we can calculate the asymptotics of $\Psi_i$'s and $\Phi_i$'s. 

\begin{proposition}\label{prop:asymptotics_barT_A^2}
Follow the assumptions of Proposition \ref{prop:concentration_lag}. For any $\lambda \in \Lambdat$, denote $\bq_T(\lambda, \bpsi) = (\mu_\star^2 - \lambda\psi_1, \mu_1^2, 0, 0,0)$, then we have 
\[
\begin{aligned}
\E_{\beps, \bbeta}[\Psi_1] \stackrel{\P}{\to}&~ \mu_1^2 \normf_1^2 \cdot  \partial_{s_2} g(0_+; \bq_T(\lambda, \bpsi); \bpsi), \\
\E_{\beps, \bbeta}[\Psi_2] \stackrel{\P}{\to}&~ \normf_1^2 \cdot \partial_{p} g(0_+; \bq_T(\lambda, \bpsi); \bpsi), \\
\E_{\beps, \bbeta}[\Psi_3] \stackrel{\P}{\to}&~ \normf_1^2  \cdot \partial_{t_2} g(0_+; \bq_T(\lambda, \bpsi); \bpsi) + \tau^2 \cdot \partial_{t_1} g(0_+; \bq_T(\lambda, \bpsi); \bpsi), \\
\E_{\beps, \bbeta}[\Phi_1] \stackrel{\P}{\to}&~ - \mu_1^2 \normf_1^2 \cdot \partial_{s1}\partial_{s2}g(0_+; \bq_T(\lambda, \bpsi); \bpsi) , \\
\E_{\beps, \bbeta}[\Phi_2] \stackrel{\P}{\to}&~ -  \normf_1^2 \cdot \partial_{s1}\partial_p g(0_+; \bq_T(\lambda, \bpsi); \bpsi), \\
\E_{\beps, \bbeta}[\Phi_3] \stackrel{\P}{\to}&~ - \normf_1^2 \cdot \partial_{s_1}\partial_{t_2} g(0_+; \bq_T(\lambda, \bpsi); \bpsi)  - \tau^2 \cdot \partial_{s_1}\partial_{t_1} g(0_+; \bq_T(\lambda, \bpsi); \bpsi), \\
\end{aligned}
\]
where $\nabla_\bq^k g(0_+; \bq; \bpsi)$ for $k \in \{1, 2\}$ stands for the $k$'th derivatives (as a vector or a matrix) of $g(\bi u; \bq; \bpsi)$ with respect to $\bq$ in the $u \to 0+$ limit (with its elements given by partial derivatives)
\[
\nabla_{\bq}^k g(0_+; \bq; \bpsi) = \lim_{u \to 0_+} \nabla_{\bq}^k g(\bi u; \bq; \bpsi). 
\]
As a consequence, we have 
\[
 \E_{\beps, \bbeta}[\overline T_c(\lambda, N, n, d)] \stackrel{\P}{\to} \overline \cT(\lambda, \psi_1, \psi_2), ~~~~~ \E_{\beps, \bbeta} [\psi_1\| \overline\ba_{T, c}(\lambda)\|_2^2] \stackrel{\P}{\to} \cA_T(\lambda, \psi_1, \psi_2), 
\]
where the definitions of $\overline \cT$ and $\cA_T$ are given in Definition \ref{def:analytic_expression_overline}. Here $\stackrel{\P}{\to}$ stands for convergence in probability as $N/d \to \psi_1$ and $n/d \to \psi_2$ (with respect to the randomness of $\bX$ and $\bTheta$). 
\end{proposition}
The Proposition above suggests that $\Psi_i$ and $\Phi_i$ concentrates with respect to the randomness in $\bX$ and $\bTheta$. To complete the concentration proof, we need to show that $\Psi_i$ and $\Phi_i$ concentrates with respect to the randomness in $\bbeta$ and $\beps$.
\begin{lemma}\label{lem:concentration_beta_eps_T}
Follow the assumptions of Proposition \ref{prop:concentration_lag}. For any $\lambda \in \Lambdat$, we have 
\[
\begin{aligned}
\Var_{\beps, \bbeta}[\Psi_1], \Var_{\beps, \bbeta}[\Psi_2], \Var_{\beps, \bbeta}[\Psi_3]  =&~ o_{d, \P}(1), \\
\Var_{\beps, \bbeta}[\Phi_1], \Var_{\beps, \bbeta}[\Phi_2],\Var_{\beps, \bbeta}[\Phi_3]  =&~ o_{d, \P}(1),  
\end{aligned}
\]
so that 
\[
\Var_{\beps, \bbeta}[\overline T_c(\lambda, N, n, d)], \Var_{\beps, \bbeta}[\|  \overline\ba_{T, c}(\lambda) \|_2^2] = o_{d, \P}(1). 
\]
Here, $o_{d, \P}(1)$ stands for converges to $0$ in probability (with respect to the randomness of $\bX$ and $\bTheta$) as $N/d \to \psi_1$ and $n/d\to\psi_2$ and $d \to \infty$.
\end{lemma}
Now, combining Proposition \ref{prop:asymptotics_barT_A^2} and \ref{lem:concentration_beta_eps_T}, we have 
\[
\overline T_c(\lambda, N, n, d) \stackrel{\P}{\to} \overline \cT(\lambda, \psi_1, \psi_2), ~~~~~ \psi_1\| \overline\ba_{T, c}(\lambda)\|_2^2 \stackrel{\P}{\to} \cA_T(\lambda, \psi_1, \psi_2).
\]
The results above combined with the arguments in Appendix \ref{sec:removing_perturbation} completes the proof for the asymptotics of $\overline T$ and $\psi_1 \| \overline \ba_T(\lambda) \|_2^2$.

\subsection{Proof of Lemma \ref{lem:expectation_Psi_Phi_U} and Lemma \ref{lem:concentration_beta_eps_U}}

\begin{proof}[Proof of Lemma \ref{lem:expectation_Psi_Phi_U}]

Note that by Assumption \ref{ass:overline_U_invertable}, the matrix $\obM = \bU_c - \psi_2^{-1} \bZ^\sT \bZ - \psi_1\lambda \id_N$ is negative definite (so that it is invertible) with high probability. Moreover, whenever $\obM$ is negative definite, the matrix $\bA(\bq_U)$ for $\bq_U = (\mu_\star^2 - \lambda\psi_1, \mu_1^2, \psi_2, 0,0)$ is also invertible. In the following, we condition on this good event happens.  

From the expansion for $\bv_i$ in \eqref{eqn:vexpan}, we have 
\[
\begin{aligned}
\E_{\bbeta, \beps}  \Psi_1 =&~ \E_{\bbeta, \beps} \Big[\Trace\Big(\overline\bM^{-1} \bv \bv^\sT \Big) \Big]= \frac{1}{d}\lambda_{d, 1}(\sigma)^2 F_1^2 \cdot \Big[\Trace\Big(\overline\bM^{-1} \bTheta \bTheta^\sT\Big)\Big] = \frac{1}{d}\mu_1^2F_1^2 \Trace\left(\overline\bM^{-1}\frac{\bTheta\bTheta^\sT}{d} \right)\times (1 + o_d(1)),
\end{aligned}
\]
where we used the relation $\lambda_{d, 1} = \mu_1/\sqrt{d} \times (1 + o_d(1))$ as in Eq. (\ref{eqn:relationship_mu_lambda}). Similarly, the second term is
\[
\begin{aligned}
\E_{\bbeta, \beps} \Psi_2 =&~ -\frac{2}{\psi_2\sqrt{d}} \E_{\bbeta, \beps} \Big[\Trace\Big( \bZ\overline{\bM}^{-1} \bv \by^\sT \Big) \Big] \\
=&~ -\frac{2}{\psi_2 d\sqrt{d}} \lambda_{d, 1}(\sigma)F_1^2 \cdot \Trace\Big( \bZ \overline{\bM}^{-1} \bTheta \bX^\sT \Big) \\
=&~ -\frac{2}{\psi_2 d^2} \mu_1 F_1^2 \cdot  \Trace\Big( \bZ \overline{\bM}^{-1} \bTheta \bX^\sT \Big)\times (1 + o_d(1)).
\end{aligned}
\]
To compute $\Psi_3$, note we have 
\[
\E_{\bbeta, \beps}[\by\by^\sT] = F_1^2 \cdot (\bX\bX^\sT)/d + \tau^2 \id_n.
\]
This gives the expansion for $\Psi_3$
\[
\begin{aligned}
\E_{\bbeta, \beps} \Psi_3 =&~ \psi_2^{-2}d^{-1}\E_{\bbeta, \beps} \Trace\Big( \bZ\overline{\bM}^{-1}\bZ^\sT \by\by^\sT \Big)\\
=&~ \psi_2^{-2}d^{-2}F_1^2\Trace\Big(\bZ\overline{\bM}^{-1}\bZ^\sT \bX\bX^\sT\Big) +  \psi_2^{-2}d^{-1} \Trace\Big(\bZ\overline{\bM}^{-1}\bZ\Big)\tau^2.\\
\end{aligned}
\]
Through the same algebraic manipulation above, we have
\[
\begin{aligned}
\E_{\bbeta, \beps} \Phi_1 =&~ \frac{1}{d}\mu_1^2F_1^2 \Trace\left(\overline\bM^{-2}\frac{\bTheta\bTheta^\sT}{d} \right)\times (1 + o_d(1)), \\
\E_{\bbeta, \beps} \Phi_2 =&~ -\frac{2}{\psi_2 d^2} \mu_1 F_1^2 \cdot \Trace\Big( \bZ \overline{\bM}^{-2} \bTheta \bX^\sT \Big)\times (1 + o_d(1)),\\
\E_{\bbeta, \beps} \Phi_3 =&~ \psi_2^{-2}d^{-2}F_1^2\cdot \Trace\Big( \bZ\overline{\bM}^{-2}\bZ^\sT \bX\bX^\sT\Big) + \psi_2^{-2}d^{-1}\tau^2 \Trace\Big( \bZ\overline{\bM}^{-2}\bZ^\sT \Big).
\end{aligned}
\]
Next, we express the trace of matrices products as the derivative of the function $G_d(\xi, \bq)$ (c.f. Definition \ref{def:log_determinant_A}). The derivatives of $G_d$ are (which can we well-defined at $\bq = \bq_U = (\mu_\star^2 - \lambda\psi_1, \mu_1^2, \psi_2, 0,0)$ with high probability by Assumption \ref{ass:overline_U_invertable})
\begin{equation}\label{eqn:derivative_log_determinant}
\begin{aligned}
\partial_{q_i} G_d(0, \bq) =  \frac{1}{d} \Trace(\bA(\bq)^{-1}\partial_i\bA(\bq)),~~~~~
\partial_{q_i} \partial_{q_j} G_d(0, \bq) =  - \frac{1}{d} \Trace(\bA(\bq)^{-1}\partial_{q_i}\bA(\bq)\bA(\bq)^{-1}\partial_{q_j}\bA(\bq)).
\end{aligned}
\end{equation}
As an example, we consider evaluating $\partial_{s_2} G_d(0, \bq)$ at $\bq = \bq_U \equiv (\mu_\star^2 - \lambda\psi_1, \mu_1^2, \psi_2, 0,0)$. Using the formula for block matrix inversion, we have
\[
\bA(\mu_\star^2 - \lambda\psi_1, \mu_1^2, \psi_2, 0,0)^{-1} = 
\begin{bmatrix}
(\mu_\star^2-\lambda \psi_1)\id_N + \mu_1^2\bQ & \bZ^\sT\\
\bZ & \psi_2\id_n
\end{bmatrix}^{-1} = 
\begin{bmatrix}
(\bU_c - \psi_2^{-1}\bZ^\sT \bZ - \psi_1 \lambda \Id_N)^{-1} & \cdots\\
\cdots & \cdots
\end{bmatrix}.
\]
Then we have
\[
\partial_{s_2} G_d(0, \bq_U) = \frac{1}{d}\Trace\left(
\begin{bmatrix}
\overline{\bM}^{-1} & \cdots\\
\cdots & \cdots
\end{bmatrix} \begin{bmatrix}
\bQ & \bzero\\
\bzero & \bzero
\end{bmatrix}\right)
=\Trace(\overline{\bM}^{-1}\bQ)/d .
\]
Applying similar argument to compute other derivatives, we get
\begin{enumerate}
\item $\Trace(\overline{\bM}^{-1}\bTheta\bTheta^\sT)/d^2=\Trace(\overline{\bM}^{-1}\bQ)/d = \partial_{s_2} G_d(0, \bq_U)$.
\item $\mu_1\cdot\Trace(\bZ\overline{\bM}^{-1}\bTheta\bX^\sT)/d^2 = \Trace(\overline{\bM}^{-1}\bZ_1^\sT\bZ)/d=- \psi_2 \partial_{p} G_d(0, \bq_U) / 2$.
\item $\Trace(\bZ\overline{\bM}^{-1}\bZ^\sT\bX\bX^\sT)/d^2 = \Trace(\bZ\overline{\bM}^{-1}\bZ^\sT\bH)/d = \psi_2^2 \partial_{t_2} G_d(0, \bq_U)-\psi_2^2$.
\item $\Trace(\bZ\overline{\bM}^{-1}\bZ^\sT)/d = \psi_2^2 \partial_{t_1} G_d(0, \bq_U) - \psi_2^2$ .
\item $\Trace(\overline{\bM}^{-2}\bQ)/d = -\partial_{s_1}\partial_{s_2}G_d(0, \bq_U)$.
\item $(2/d\psi_2)\cdot\Trace(\bZ_1^{\sT}\bZ\overline{\bM}^{-2}) = \partial_{s_1}\partial_p G_d(0, \bq_U)$.
\item $\Trace(\overline{\bM}^{-2}\bZ^\sT\bH\bZ)/(d\psi_2^2)=-\partial_{s_1}\partial_{t_2}G_d(0, \bq_U)$.
\item $\Trace(\overline{\bM}^{-2}\bZ^\sT\bZ)/(d\psi_2^2)=-\partial_{s_1}\partial_{t_1}G_d(0, \bq_U)$.
\end{enumerate}
Combining these equations concludes the proof. 
\end{proof}

\begin{proof}[Proof of Lemma \ref{lem:concentration_beta_eps_U}]

We prove this lemma by assuming that $\bbeta$ follows a different distribution: $\bbeta \sim \cN(\bzero, (\| \normf_1 \|_2^2 / d) \id_d)$. The case when $\bbeta \sim \Unif(\S^{d-1}(\normf_1))$ can be treated similarly. 

By directly calculating the variance, we can show that, there exists scalers $(c_{ik}^{(d)})_{k \in [K_i]}$ with $c_{ik}^{(d)} = \Theta_{d}(1)$, and matrices $(\bA_{ik}, \bB_{ik})_{k \in [K_i]} \subseteq \{ \id_N, \bQ, \bZ^\sT \bH \bZ, \bZ^\sT \bZ  \}$, such that the variance of $\Psi_i$'s can be expressed in form
\[
\Var_{\beps, \bbeta}(\Psi_i) = \frac{1}{d} \sum_{k = 1}^{K_i} c_{ik}^{(d)} \Trace(\obM^{-1} \bA_{ik} \obM^{-1} \bB_{ik}) / d. 
\]

For example, by Lemma \ref{lem:variance_calculations}, we have
\[
\begin{aligned}
\Var_{\bbeta \sim \cN(\bzero, (\normf_1^2/d) \id_d)}(\Psi_1) =&~ \lambda_{d, 1}(\sigma)^4 \Var_{\bbeta \sim \cN(\bzero, (\normf_1^2/d) \id_d))}(\bbeta^\sT \bTheta^\sT \obM^{-1} \bTheta \bbeta) = 2 \lambda_{d, 1}(\sigma)^4 \normf_1^4 \| \bTheta^\sT \obM^{-1} \bTheta \|_F^2 / d^2 \\
=&~ c_{1}^{(d)} \Trace( \obM^{-1} \bQ \obM^{-1} \bQ ) /d^2,
\end{aligned}
\]
where $c_{1}^{(d)} = 2 d^2 \lambda_{d, 1}(\sigma)^4 \normf_1^4 = O_d(1)$. The variance of $\Psi_2$ and $\Psi_3$ can be calculated similarly.

Note that each $\Trace(\obM^{-1} \bA_{ik} \obM^{-1} \bB_{ik}) / d$ can be expressed as an entry of $\nabla_\bq^2  G_d(0; \bq)$ (c.f. Eq. (\ref{eqn:derivative_log_determinant})), and by Proposition \ref{prop:exchange_limit_g_U}, they are of order $O_{d, \P}(1)$. This gives 
\[
\Var_{\beps, \bbeta}(\Psi_i) = o_{d, \P}(1). 
\]

Similarly, for the same set of scalers $(c_{ik}^{(d)})_{k \in [K_i]}$ and matrices $(\bA_{ik}, \bB_{ik})_{k \in [K_i]}$, we have
\[
\Var_{\beps, \bbeta}(\Phi_i) = \frac{1}{d} \sum_{k = 1}^{K_i} c_{ik} \Trace(\obM^{-2} \bA_{ik} \obM^{-2} \bB_{ik}) / d. 
\]
Note that for two semidefinite matrices $\bA, \bB$, we have $\Trace(\bA \bB) \le \| \bA \|_{\op} \Trace(\bB)$. Moreover, note we have $\| \obM \|_{\op} = O_{d, \P}(1)$ (by Assumption \ref{ass:overline_U_invertable}). This gives 
\[
\Var_{\beps, \bbeta}(\Phi_i) = o_{d, \P}(1). 
\]
This concludes the proof. 
\end{proof}

\subsection{Auxiliary Lemmas}\label{sec:auxiliary_lemmas}

The following lemma (Lemma \ref{lem:gegenbauer_identity}) is a reformulation of Proposition 3 in \cite{ghorbani2019linearized}. We present it in a stronger form, but it can be easily derived from the proof of Proposition 3 in \cite{ghorbani2019linearized}. This lemma was first proved in \cite{el2010spectrum} in the Gaussian case.
(Notice that the second estimate ---on $Q_k(\bTheta \bX^\sT)$---  follows
by applying the first one whereby $\bTheta$ is replaced by $\bW = [\bTheta^{\sT}|\bX^{\sT}]^{\sT}$
\begin{lemma}\label{lem:gegenbauer_identity}
Let $\bTheta = (\btheta_1, \ldots, \btheta_N)^\sT \in \R^{N \times d}$ with $(\btheta_a)_{a\in [N]} \sim_{iid} \Unif(\S^{d- 1}(\sqrt d))$ and $\bX = (\bx_1, \ldots, \bx_n)^\sT \in \R^{n \times d}$ with $(\bx_i)_{i\in [n]} \sim_{iid} \Unif(\S^{d- 1}(\sqrt d))$. Assume $1/c \le n / d, N / d \le c$ for some constant $c \in (0, \infty)$. Then 
\begin{align}
  \E\Big[ \sup_{k \ge 2} \| Q_k(\bTheta \bTheta^\sT) - \id_N \|_{\op}^2 \Big]&= o_d(1)\, ,\label{eq:QTT}\\
  \E\Big[ \sup_{k \ge 2} \| Q_k(\bTheta \bX^\sT) \|_{\op}^2 \Big] &= o_d(1). \label{eq:QTX}
\end{align}
\end{lemma}
Notice that the second estimate ---on $Q_k(\bTheta \bX^\sT)$---  follows
by applying the first one ---Eq.~\eqref{eq:QTT}--- whereby $\bTheta$ is replaced by
$\bW = [\bTheta^{\sT}|\bX^{\sT}]^{\sT}$, and we use
$\| Q_k(\bTheta \bX^\sT) \|_{\op} \le \| Q_k(\bW \bW^\sT)-\id_{N+n} \|_{\op}$. 

The following lemma (Lemma \ref{lem:decomposition_of_kernel_matrix}) can be easily derived from Lemma \ref{lem:gegenbauer_identity}. Again, this lemma was first proved in \cite{el2010spectrum} in the Gaussian case.  
\begin{lemma}\label{lem:decomposition_of_kernel_matrix}
Let $\bTheta = (\btheta_1, \ldots, \btheta_N)^\sT \in \R^{N \times d}$ with $(\btheta_a)_{a\in [N]} \sim_{iid} \Unif(\S^{d- 1}(\sqrt d))$. Let activation function $\sigma$ satisfies Assumption \ref{ass:activation}. Assume $1/c \le N / d \le c$ for some constant $c \in (0, \infty)$. Denote 
\[
\bU = \Big(\E_{\bx \sim \Unif(\S^{d-1}(\sqrt d))}[\sigma(\< \btheta_a, \bx\> / \sqrt d) \sigma(\< \btheta_b, \bx \> / \sqrt d)] \Big)_{a, b \in [N]} \in \R^{N \times N}. 
\]
Then we can rewrite the matrix $\bU$ to be
\[
\bU =  \lambda_{d, 0}(\sigma)^2 \ones_N \ones_N^\sT + \ob_1^2 \bQ  + \ob_\star^2 (\id_N + \bDelta),
\]
with $\bQ = \bTheta \bTheta^\sT / d$ and $\E[\| \bDelta \|_{\op}^2] = o_{d}(1)$. 
\end{lemma}

In the following, we show that, under sufficient regularity condition of $\sigma$, we have $\lambda_{d, 0}(\sigma) = O(1/d)$. 

\begin{lemma}\label{lem:small_lambda_d0}
Let $\sigma \in C^2(\R)$ with $\vert \sigma'(x)\vert, \vert \sigma''(x)\vert < c_0 e^{c_1 \vert x \vert}$ for some $c_0, c_1 \in \R$. Assume that $\E_{G \sim \cN(0, 1)}[\sigma(G)] = 0$. Then we have 
\[
\lambda_{d, 0}(\sigma) \equiv \E_{\bx \sim \Unif(\S^{d-1}(\sqrt{d}))}[\sigma(x_1)] = O(1/d).
\]
\end{lemma}

\begin{proof}[Proof of Lemma \ref{lem:small_lambda_d0}]

Let $\bx \sim \Unif(\S^{d-1}(\sqrt{d}))$ and $\gamma \sim \chi(d) / \sqrt{d}$ independently. Then we have $\gamma \bx  \sim \cN(\bzero, \id_d)$, so that by the assumption, we have $\E[\sigma(\gamma x_1)] = 0$. 

As a consequence, by the second order Taylor expansion, and by the independence of $\gamma$ and $\bx$, we have (for $\xi(x_1) \in [\gamma, 1]$) 
\[
\begin{aligned}
\vert \lambda_{d, 0}(\sigma) \vert =&~ \vert \E[\sigma(x_1)]\vert \le \vert \E[\sigma(x_1)] - \E[\sigma(\gamma x_1)] \vert \le \Big\vert \E[\sigma'(x_1)x_1] \E[\gamma - 1] \Big\vert + \Big\vert (1/2)\E[\sigma''(\xi(x_1) x_1) (\gamma - 1)^2] \Big\vert\\
\le&~ \Big\vert \E[\sigma'(x_1)x_1] \Big\vert \cdot \Big\vert \E[\gamma - 1] \Big\vert +  (1/2)\E\Big[ \sup_{u \in [\gamma, 1]} \sigma''(u x_1)^2 \Big]^{1/2}  \E[ (\gamma - 1)^4]^{1/2}. 
\end{aligned}
\]
By the assumption that $\vert \sigma'(x)\vert, \vert \sigma''(x) \vert < c_0 e^{c_1 \vert x \vert}$ for some $c_0, c_1 \in \R$, there exists constant $K$ that only depends on $c_0$ and $c_1$ such that \[
\sup_{d} \Big\vert \E[\sigma'(x_1)x_1] \Big\vert \le K,~~~~~ \sup_{d}\Big\vert (1/2)\E\Big[ \sup_{u \in [\gamma, 1]} \sigma''(u x_1)^2 \Big]^{1/2} \Big\vert \le K.
\]
Moreover, by property of the $\chi$ distribution, we have 
\[
\vert \E[\gamma - 1] \vert = O(d^{-1}), ~~~~~~\E[ (\gamma - 1)^4]^{1/2} = O(d^{-1}).
\]
This concludes the proof. 
\end{proof}

The following lemma is a simple variance calculation and can be found as Lemma C.5 in \cite{mm19}. We restate here for completeness. 

\def\bg{{\boldsymbol g}}
\def\bB{{\boldsymbol B}}

\begin{lemma}\label{lem:variance_calculations}
Let $\bA \in \R^{n \times N}$ and $\bB \in \R^{n \times n}$. Let $\bg = (g_1, \ldots, g_n)^\sT$ with $g_i \sim_{iid} \P_g$, $\E_g[g] = 0$, and $\E_g[g^2] = 1$. Let $\bh = (h_1, \ldots, h_N)^\sT$ with $h_i \sim_{iid} \P_h$, $\E_h[h] = 0$, and $\E_h[h^2] = 1$. Further we assume that $\bh$ is independent of $\bg$. Then we have 
\[
\begin{aligned}
\Var(\bg^\sT \bA \bh) =&~ \| \bA \|_F^2, \\
\Var(\bg^\sT \bB \bg) =&~ \sum_{i = 1}^n B_{ii}^2 (\E[g^4] - 3) + \| \bB \|_F^2 + \Trace(\bB^2). 
\end{aligned}
\]
\end{lemma}

\section{Proof of Theorem \ref{thm:main_theorem}}\label{sec:proof_main_thm}

Here we give the whole proof for $U$. The proof for $T$ is the same.

For fixed $A^2 \in \Gamma_U \equiv \{\cA_U(\lambda, \psi_1, \psi_2): \lambda \in \Lambdau \}$, we denote 
\[
\lambda_\star(A^2) = \inf_{\lambda} \Big\{ \lambda:  \cA_U(\lambda, \psi_1, \psi_2) = A^2 \Big\}. 
\]
By the definition of $\Gamma_U$, the set $\{ \lambda:  \cA_U(\lambda, \psi_1, \psi_2) = A^2 \}$ is non-empty and lower bounded, so that $\lambda_\star(A^2)$ can be well-defined. Moreover, we have $\lambda_\star(A^2) \in \Lambdau$. It is also easy to see that we have 
\begin{equation}\label{eqn:lambda_star_U_Asquare_min}
\lambda_\star(A^2) \in \argmin_{\lambda \ge 0} \Big[ \overline \cU(\lambda, \psi_1, \psi_2) + \lambda A^2 \Big]. 
\end{equation}

\subsection{Upper bound} 

Note we have 
\[
\begin{aligned}
U(A, N, n, d) =&~ \sup_{(N/d) \| \ba \|_2^2 \le A^2} \Big( R(\ba) - \what R_n(\ba) \Big) \\
\le&~ \inf_{\lambda} \sup_{(N/d) \| \ba \|_2^2 \le A^2} \Big( R(\ba) - \what R_n(\ba) - \psi_1\lambda (\|\ba\|_2^2 - \psi_1^{-1}A^2) \Big) \\
\le&~ \inf_{\lambda} \Big[ \overline U(\lambda, N, n, d) + \lambda A^2  \Big] \\
\le &~  \overline U(\lambda_\star(A^2), N, n, d) + \lambda_\star(A^2) A^2. 
\end{aligned}
\]
Note that $\lambda_\star(A^2) \in \Lambdau$, so by Lemma \ref{prop:asymptotics_barU_A^2}, in the limit of Assumption \ref{ass:linear}, we have
\[
U(A, N, n, d) \le  \overline \cU(\lambda_\star(A^2), \psi_1, \psi_2) + \lambda_\star(A^2) A^2  + o_{d, \P}(1) = \cU(A, \psi_1, \psi_2) + o_{d, \P}(1), 
\]
where the last equality is by Eq. (\ref{eqn:lambda_star_U_Asquare_min}). This proves the upper bound. 

\subsection{Lower bound}
For any $A^2 > 0$, we define a random variable $\hat \lambda(A^2)$ (which depend on $\bX$, $\bTheta$, $\bbeta$, $\beps$) by 
\[
\hat \lambda(A^2) = \inf\Big\{ \lambda: \lambda \in  \argmin_{\lambda \ge 0} \Big[ \overline U(\lambda, N, n, d) + \lambda A^2 \Big]  \Big\}. 
\]
By Proposition \ref{prop:strong_duality}, the set is should always be non-empty, so that $\hat \lambda(A^2)$ can always be well-defined. 

Moreover, since $\lambda_\star(A^2) \in \Lambdau$, by Assumption \ref{ass:overline_U_invertable}, as we have shown in the proof in Proposition \ref{prop:concentration_lag}, we can uniquely define $\overline\ba_{U}(\lambda_\star(A^2))$ with high probability, where
\[
\overline\ba_{U}(\lambda_\star(A^2)) = \argmax_{\ba} \Big[ R(\ba) - \what R_n(\ba) - \psi_1 \lambda_\star(A^2) \| \ba \|_2^2\Big]. 
\]
As a consequence, for a small $\eps > 0$, the following event $\cE_{\eps, d}$ can be well-defined with high probability
\[
\begin{aligned}
\cE_{\eps, d} =&~ \Big\{ \psi_1 \| \overline\ba_{U} (\lambda_\star(A^2)) \|_2^2 \ge A^2 - \eps \Big\} \cap \Big\{ \hat \lambda(A^2 + \eps) \le \lambda_\star(A^2) \Big\}\\
=&~  \Big\{ A^2 - \eps\le \psi_1 \| \overline\ba_{U} (\lambda_\star(A^2)) \|_2^2 \le  A^2 + \eps \Big\}. 
\end{aligned}
\]

Now, by Proposition \ref{prop:concentration_lag}, in the limit of Assumption \ref{ass:linear}, we have
\begin{equation}\label{eqn:event_Asquare_concentration}
\lim_{d \to \infty} \P_{\bX, \bTheta, \bbeta, \beps}(\cE_{\eps, d}) = 1, 
\end{equation}
and we have 
\begin{equation}\label{eqn:overline_U_concentration}
\begin{aligned}
\cU(\lambda_\star(A^2), \psi_1, \psi_2)  = \overline \cU(\lambda_\star(A^2), \psi_1, \psi_2) + o_{d, \P}(1). 
\end{aligned}
\end{equation}

By the strong duality as in Proposition \ref{prop:strong_duality}, for any $A^2 \in \Gamma_U$, we have 
\[
\begin{aligned}
U(A, N, n, d) =&~  \overline U(\hat \lambda(A^2), N, n, d) + \hat \lambda(A^2) A^2 . 
\end{aligned}
\]
Consequently, for small $\eps >0$, when the event $\cE_{\eps, d}$ happens, we have
\[
\begin{aligned}
&~ U( (A^2 + \eps)^{1/2}, N, n, d) \\
=&~  \sup_{\ba} \Big( R(\ba) - \what{R}_n(\ba) - \psi_1\hat \lambda(A^2 + \eps) \cdot \big(\| \ba \|_2^2 - \psi_1^{-1}(A^2+ \eps)\big)\Big)\\
\ge&~ R(\overline\ba_{U}(\lambda_\star(A^2))) - \what{R}_n(\overline\ba_{U}(\lambda_\star(A^2))) - \psi_1\hat \lambda(A^2 + \eps) \cdot \big(\|\overline\ba_{U}(\lambda_\star(A^2)) \|_2^2 - \psi_1^{-1}(A^2 + \eps)\big) \\
\ge&~ R(\overline\ba_{U}(\lambda_\star(A^2))) - \what{R}_n(\overline\ba_{U}(\lambda_\star(A^2))) - \psi_1\hat \lambda(A^2 + \eps) \cdot \big(\|\overline\ba_{U}(\lambda_\star(A^2)) \|_2^2 - \psi_1^{-1}(A^2 - \eps)\big) \\
\ge&~ R(\overline\ba_{U}(\lambda_\star(A^2))) - \what{R}_n(\overline\ba_{U}(\lambda_\star(A^2))) - \psi_1 \lambda_\star(A^2) \cdot \big(\|\overline\ba_{U}(\lambda_\star(A^2)) \|_2^2 - \psi_1^{-1}(A^2 - \eps)\big)  \\
=&~  \overline U( \lambda_\star(A^2), N, n, d)+ \lambda_\star(A^2) \cdot (A^2 - \eps). 
\end{aligned}
\]
As a consequence, by Eq. (\ref{eqn:event_Asquare_concentration}) and (\ref{eqn:overline_U_concentration}), we have 
\[
U( (A^2 + \eps)^{1/2}, N, n, d) \ge \overline \cU( \lambda_\star(A^2), \psi_1, \psi_2) + \lambda_\star(A^2) \cdot (A^2 - \eps) - o_{d, \P}(1)= \cU(A, \psi_1, \psi_2) - \eps \lambda_\star(A^2)  - o_{d, \P}(1) . 
\]
where the last equality is by the definition of $\cU$ as in Definition \ref{def:formula_U_T}, and by the fact that $\lambda_\star(A^2) \in \argmin_{\lambda \ge 0} [ \overline \cU(\lambda, \psi_1, \psi_2) + \lambda A^2 ]$. Taking $\eps$ sufficiently small proves the lower bound. This concludes the proof of Theorem \ref{thm:main_theorem}. 

\section{Technical background}
\label{sec:Background}

In this section we introduce additional technical background useful for the proofs. 
In particular, we will use decompositions in (hyper-)spherical harmonics on the  $\S^{d-1}(\sqrt{d})$ and in Hermite polynomials on the real line. We refer the readers to \cite{costas2014spherical,szego1939orthogonal,chihara2011introduction,ghorbani2019linearized, mm19} for further information on these topics. 

\subsection{Functional spaces over the sphere}

For $d \ge 1$, we let $\S^{d-1}(r) = \{\bx \in \R^{d}: \| \bx \|_2 = r\}$ denote the sphere with radius $r$ in $\reals^d$.
We will mostly work with the sphere of radius $\sqrt d$, $\S^{d-1}(\sqrt{d})$ and will denote by $\gamma_d$  the uniform probability measure on $\S^{d-1}(\sqrt d)$. 
All functions in the following are assumed to be elements of $ L^2(\S^{d-1}(\sqrt d) ,\gamma_d)$, with scalar product and norm denoted as $\<\,\cdot\,,\,\cdot\,\>_{L^2}$
and $\|\,\cdot\,\|_{L^2}$:
\begin{align}
\<f,g\>_{L^2} \equiv \int_{\S^{d-1}(\sqrt d)} f(\bx) \, g(\bx)\, \gamma_d(\de \bx)\,.
\end{align}

For $\ell\in\integers_{\ge 0}$, let $\tilde{V}_{d,\ell}$ be the space of homogeneous harmonic polynomials of degree $\ell$ on $\reals^d$ (i.e. homogeneous
polynomials $q(\bx)$ satisfying $\Delta q(\bx) = 0$), and denote by $V_{d,\ell}$ the linear space of functions obtained by restricting the polynomials in $\tilde{V}_{d,\ell}$
to $\S^{d-1}(\sqrt d)$. With these definitions, we have the following orthogonal decomposition
\begin{align}
L^2(\S^{d-1}(\sqrt d) ,\gamma_d) = \bigoplus_{\ell=0}^{\infty} V_{d,\ell}\, . \label{eq:SpinDecomposition}
\end{align}
The dimension of each subspace is given by
\begin{align}
\dim(V_{d,\ell}) = B(d, \ell) = \frac{2 \ell + d - 2}{\ell} { \ell + d - 3 \choose \ell - 1} \, .
\end{align}
For each $\ell\in \integers_{\ge 0}$, the spherical harmonics $\{ Y_{\ell j}^{(d)}\}_{1\le j \le B(d, \ell)}$ form an orthonormal basis of $V_{d,\ell}$:
\[
\<Y^{(d)}_{ki}, Y^{(d)}_{sj}\>_{L^2} = \delta_{ij} \delta_{ks}.
\]
Note that our convention is different from the more standard one, that defines the spherical harmonics as functions on $\S^{d-1}(1)$.
It is immediate to pass from one convention to the other by a simple scaling. We will drop the superscript $d$ and write $Y_{\ell, j} = Y_{\ell, j}^{(d)}$ whenever clear from the context.

We denote by $\proj_k$  the orthogonal projections to $V_{d,k}$ in $L^2(\S^{d-1}(\sqrt d),\gamma_d)$. This can be written in terms of spherical harmonics as
\begin{align}
\proj_k f(\bx) \equiv&~ \sum_{l=1}^{B(d, k)} \< f, Y_{kl}\>_{L^2} Y_{kl}(\bx). 
\end{align}
Then for a function $f \in L^2(\S^{d-1}(\sqrt d))$, we have 
\[
f(\bx) = \sum_{k = 0}^\infty \proj_k f(\bx) = \sum_{k = 0 }^\infty \sum_{l = 1}^{B(d, k)} \< f, Y_{kl}\>_{L^2} Y_{kl}(\bx). 
\]

\subsection{Gegenbauer polynomials}
\label{sec:Gegenbauer}

The $\ell$-th Gegenbauer polynomial $Q_\ell^{(d)}$ is a polynomial of degree $\ell$. Consistently
with our convention for spherical harmonics, we view $Q_\ell^{(d)}$ as a function $Q_{\ell}^{(d)}: [-d,d]\to \reals$. The set $\{ Q_\ell^{(d)}\}_{\ell\ge 0}$ forms an orthogonal basis on $L^2([-d,d], \tilde \tau_d)$ (where $\tilde \tau_d$ is the distribution of $\<\bx_1, \bx_2\>$ when $\bx_1, \bx_2 \sim_{i.i.d.} \Unif(\S^{d-1}(\sqrt d))$), satisfying the normalization condition:
\begin{align}
\< Q^{(d)}_k, Q^{(d)}_j \>_{L^2(\tilde \tau_d)} = \frac{1}{B(d,k)}\, \delta_{jk} \, .  \label{eq:GegenbauerNormalization}
\end{align}
In particular, these polynomials are normalized so that  $Q_\ell^{(d)}(d) = 1$. 
As above, we will omit the superscript $d$ when clear from the context (write it as $Q_\ell$ for notation simplicity).

Gegenbauer polynomials are directly related to spherical harmonics as follows. Fix $\bv\in\S^{d-1}(\sqrt{d})$ and 
consider the subspace of  $V_{\ell}$ formed by all functions that are invariant under rotations in $\reals^d$ that keep $\bv$ unchanged.
It is not hard to see that this subspace has dimension one, and coincides with the span of the function $Q_{\ell}^{(d)}(\<\bv,\,\cdot\,\>)$.

We will use the following properties of Gegenbauer polynomials
\begin{enumerate}
\item For $\bx, \by \in \S^{d-1}(\sqrt d)$
\begin{align}
\< Q_j^{(d)}(\< \bx, \cdot\>), Q_k^{(d)}(\< \by, \cdot\>) \>_{L^2(\S^{d-1}(\sqrt d), \gamma_d)} = \frac{1}{B(d,k)}\delta_{jk}  Q_k^{(d)}(\< \bx, \by\>).  \label{eq:ProductGegenbauer}
\end{align}
\item For $\bx, \by \in \S^{d-1}(\sqrt d)$
\begin{align}
Q_k^{(d)}(\< \bx, \by\> ) = \frac{1}{B(d, k)} \sum_{i =1}^{ B(d, k)} Y_{ki}^{(d)}(\bx) Y_{ki}^{(d)}(\by). \label{eq:GegenbauerHarmonics}
\end{align}
\end{enumerate}
Note in particular that property 2 implies that --up to a constant-- $Q_k^{(d)}(\< \bx, \by\> )$ is a representation of the projector onto 
the subspace of degree-$k$ spherical harmonics
\begin{align}
(\proj_k f)(\bx) = B(d,k) \int_{\S^{d-1}(\sqrt{d})} \, Q_k^{(d)}(\< \bx, \by\> )\,  f(\by)\, \gamma_d(\de\by)\, .\label{eq:ProjectorGegenbauer}
\end{align}
For a function $\sigma \in L^2([-\sqrt d, \sqrt d], \tau_d)$ (where $\tau_d$ is the distribution of $\< \bx_1, \bx_2 \> / \sqrt d$ when $\bx_1, \bx_2 \sim_{iid} \Unif(\S^{d-1}(\sqrt d))$), denoting its spherical harmonics coefficients $\lambda_{d, k}(\sigma)$ to be 
\begin{align}\label{eqn:technical_lambda_sigma}
\lambda_{d, k}(\sigma) = \int_{[-\sqrt d , \sqrt d]} \sigma(x) Q_k^{(d)}(\sqrt d x) \tau_d(x),
\end{align}
then we have the following equation holds in $L^2([-\sqrt d, \sqrt d],\tau_d)$ sense
\begin{equation}\label{eqn:sigma_G_decomposition}
\sigma(x) = \sum_{k = 0}^\infty \lambda_{d, k}(\sigma) B(d, k) Q_k^{(d)}(\sqrt d x). 
\end{equation}

\subsection{Hermite polynomials}

The Hermite polynomials $\{\He_k\}_{k\ge 0}$ form an orthogonal basis of $L^2(\reals,\mu_G)$, where $\mu_G(\de x) = e^{-x^2/2}\de x/\sqrt{2\pi}$ 
is the standard Gaussian measure, and $\He_k$ has degree $k$. We will follow the classical normalization (here and below, expectation is with respect to
$G\sim\normal(0,1)$):
\begin{align}
\E\big\{\He_j(G) \,\He_k(G)\big\} = k!\, \delta_{jk}\, .
\end{align}
As a consequence, for any function $\sigma \in L^2(\reals,\mu_G)$, we have the decomposition
\begin{align}\label{eqn:sigma_He_decomposition}
\sigma(x) = \sum_{k=1}^{\infty}\frac{\mu_k(\sigma )}{k!}\, \He_k(x)\, ,\;\;\;\;\;\; \mu_k(\sigma) \equiv \E\big\{\sigma(G)\, \He_k(G)\}\, .
\end{align}

The Hermite polynomials can be obtained as high-dimensional limits of the Gegenbauer polynomials introduced in the previous section. Indeed, 
the Gegenbauer polynomials (up to a $\sqrt d$ scaling in domain) are constructed by Gram-Schmidt orthogonalization of the monomials $\{x^k\}_{k\ge 0}$ with respect to the measure 
$\tau_d$, while Hermite polynomial are obtained by Gram-Schmidt orthogonalization with respect to $\mu_G$. Since $\tau_d\Rightarrow \mu_G$
(here $\Rightarrow$ denotes weak convergence),
it is immediate to show that, for any fixed integer $k$, 
\begin{align}
\lim_{d \to \infty} \Coeff\{ Q_k^{(d)}( \sqrt d x) \, B(d, k)^{1/2} \} = \Coeff\left\{ \frac{1}{(k!)^{1/2}}\,\He_k(x) \right\}\, .\label{eq:Gegen-to-Hermite}
\end{align}
Here and below, for $P$ a polynomial, $\Coeff\{ P(x) \}$ is  the vector of the coefficients of $P$. As a consequence, for any fixed integer $k$, we have 
\begin{align}\label{eqn:relationship_mu_lambda}
\mu_k(\sigma) = \lim_{d \to \infty} \lambda_{d, k}(\sigma) (B(d, k) k!)^{1/2},
\end{align}
where $\mu_k(\sigma)$ and $\lambda_{d, k}(\sigma)$ are given in Eq. (\ref{eqn:sigma_He_decomposition}) and (\ref{eqn:technical_lambda_sigma}).